\documentclass[review,11pt]{ReportTemplate}
\usepackage{bm}
\usepackage{graphicx}
\usepackage{paralist,algorithmic,algorithm}
\usepackage{amsmath,mathrsfs,amssymb}
\newtheorem{lemma}{Lemma} 
\newtheorem{thm}{Theorem} 

\theoremstyle{definition}
\newtheorem{defn}{Definition}

\theoremstyle{remark}
\theoremstyle{definition}
\newenvironment{proof}{\textbf{Proof:}\ }{\hspace{\stretch{1}}$\square$\\}

\begin{document}
\begin{frontmatter}
\title{Distribution-Free One-Pass Learning}

\author{Peng Zhao}
\author{Zhi-Hua Zhou\corref{cor1}}
\address{National Key Laboratory for Novel Software Technology\\
Nanjing University, Nanjing 210093, China} \cortext[cor1]{\small Corresponding author.
Email: zhouzh@nju.edu.cn}

\begin{abstract} 
In many large-scale machine learning applications, data are accumulated with time, and thus, an appropriate model should be able to update in an online paradigm. Moreover, as the whole data volume is unknown when constructing the model, it is desired to scan each data item only once with a storage independent with the data volume. It is also noteworthy that the distribution underlying may change during the data accumulation procedure. To handle such tasks, in this paper we propose DFOP, a distribution-free one-pass learning approach. This approach works well when distribution change occurs during data accumulation, without requiring prior knowledge about the change. Every data item can be discarded once it has been scanned. Besides, theoretical guarantee shows that the estimate error, under a mild assumption, decreases until convergence with high probability. The performance of DFOP for both regression and classification are validated in experiments.
\end{abstract} 

\begin{keyword}
data stream \sep non-stationary learning \sep one-pass learning \sep time-series
\end{keyword}
\end{frontmatter}

\section{Introduction}
With a rapid growth in collecting data, the volume of data generated makes it a challenge for traditional machine learning approaches. The main challenges are multi-faced, in general, \textit{accumulating} and \textit{evolving} properties are two of the most troublesome issues.

For the accumulating property, apparently, it's impractical to store all the data entirely due to the limitation of memory and computation resources. Hence, an offline approach is not suitable any more in such tasks. Only in an online processing paradigm can prediction models be trained and updated incrementally. It's also worthy to mention that an online streaming approach is called one-pass when it requires going through the training data only once without storing the entire dataset. The reason to pursue one-pass property is due to the fact that sometimes the raw data is discarded or no longer accessible after being processed. A one-pass approach guarantees the learning process be independent with the volume of data stream. Apparently, it is much more demanding and difficult.

Furthermore, the evolving nature of data stream also makes it challenging to directly apply traditional machine learning approaches because it's not reasonable to assume that the current and future data are coming from the same distribution any more. In a non-stationary environment, which is very common in data generation process, the distribution underlying is likely to dynamically change over time. For instance, the clicking information collected in recommendation system is certainly evolving because customers' interests probably change when looking through the product pages. Another example is credit scoring, the criteria of credit granting should properly alter since a changing economic conditions would have a great influence on people's manner. Such phenomena above are typical examples of \textit{distribution change}. Under this scenario, the performance of traditional approaches dramatically drop down and thus are not empirically and theoretically suitable for these tasks.

To simultaneously address these two issues, in this work, we propose DFOP, a distribution free one-pass learning approach to deal with data stream emerging distribution change under one-pass constraints. The advantages of our approach are following: firstly, by recursively solving the target, we guarantee that only one time would data stream be gone through. Secondly, based on a forgetting mechanism, the loss of older data is discounted without explicitly modelling the dynamics or assuming prior information about distribution change. In streaming regression, a theoretical guarantee is presented showing that the estimate error of dynamic concept would decrease until convergence with high probability. Meanwhile, empirical experiments on both synthesis and real-world datasets indicate the effectiveness and practicability of DFOP in both regression and classification scenarios.

The rest of this paper is organized as follows. In Section \ref{sec-related}, we briefly review of some related work. Then, static scenario model is introduced in preliminaries. Next, in the non-stationary environment, DFOP is presented to handle dynamics in regression and classification scenarios. Both theoretical guarantee and empirical effectiveness have been examined. Finally, we conclude the paper.

\section{Related Work}
\label{sec-related}
\textbf{Online and One-Pass Algorithms.} With a rapid growth of data volumn and velocity, it's no longer practicable to adopt offline mode algorithm for streaming data learning tasks. Hence, online style algorithms become gradually attractive which update the current model with the most recent examples. In general, they are error driven, updating the current model depending on whether the current example is misclassified~\cite{journals/csur/GamaZBPB14}. Representative algorithms include Perceptron~\cite{journal/rosenblatt1958perceptron}, Winnow~\cite{journals/ml/Littlestone87} and many other variations. Besides, a paradigm of “prediction with expert advice”~\cite{book/Cambridge/cesa2006prediction} also inspires some interesting works, such as AddExp\cite{conf/icml/KolterM05} and DWM\cite{conf/icdm/KolterM03,journals/jmlr/KolterM07}. Most of those approaches require to store the entire or partial training data and scan data items multiple times. Recently, one-pass algorithms gradually draw more attentions demanding that each data item should be processed only once. Concretely speaking, after a data item has been processed and relevant statistics have been stored, the raw data item should be discarded and never be accessed any more. Obviously, one-pass constraints impose a higher degree of difficulty on algorithm design. Some efforts have been devoted~\cite{conf/nips/WuBSD16}.

\textbf{Nonstationary Learning.} Owing to the effectiveness and simplicity, \textit{sliding window} is usually adopted to handle data stream with distribution change. It only uses a fixed or variable number of recent data which are the most informative for current prediction~\cite{journals/ml/Littlestone87,journals/ida/Klinkenberg04}. Usually, the model built is updated following two processes: one is a learning process, i.e., updating the model based on the  new coming data; the other one is a forgetting process, i.e., discarding data items that are moving out of the window~\cite{journals/csur/GamaZBPB14}. However, how to choose an appropriate window size is of great importance which now mainly depends on heuristics to a certian extent. Some efforts have been paid to select window size adaptively~\cite{journals/ida/Klinkenberg04,conf/sbia/GamaMCR04,journals/ida/KunchevaZ09}. The common strategy to adjust window size is based on the performance or estimate of generalization error. SVM-ada~\cite{journals/ida/Klinkenberg04} presents a theoretically supported approach , however, the computational efficiency issue makes it not practical in real-world applications. 

Our proposed approach DFOP, short for Distribution-Free One-Pass, is a one-pass style algorithm, i.e., it could guarantee that only one time will data items been gone through. Besides, DFOP is distribution-free, i.e., different from those traditional approaches dealing with distribution change, we did not explicitly model the dynamics, and no prior information about distribution change is assumed. 

\section{Preliminaries}
In this part, streaming regression model in a static scenario is briefly introduced. 

In a streaming scenario, we denote a labeled dataset as $\{\mathbf x(t),y(t)\}$, where $\mathbf x(t)$ is the feature of the $t$-th instance and $y(t)$ is a real-valued output. Furthermore, we assume a linear model as follows, 
\begin{equation}
y(t) = \mathbf x(t)^{\rm T} \mathbf w(t-1)+\epsilon(t)
\end{equation}
where $\{\epsilon(t)\}$ is the noise sequence, $\{\mathbf{w}(t)\}$ is what we desire to estimate.

When in a static scenario, the sequence $\{\mathbf{w}(t)\}$ is a constant vector denoted as $\mathbf{w}_0$. Then, the least square could be adopted to minimize the residual sum of squares, which has a close-form solution.
However, it fails when adding an online/one-pass constraint which demands the raw item is discarded after it has been processed. Recursive least square (RLS) and stochastic gradient descent (SGD) are two typical approaches to solve this problem in an online paradigm.

When in a non-stationary environment, especially  when the distribution underlying changes, traditional approaches are not suitable since we could never expect the typical i.i.d assumption continue to work any longer. In the next sections, we propose to handle this scenario based on exponential forgetting mechanism without explicitly modelling the evolution of data stream, and theoretical support and empirical demonstration are presented.

In the following, $\lVert \cdot \rVert$ denotes the $\ell_2$-norm in $\mathbb{R}^n$ space. Meanwhile, for a bounded real-valued sequence $\{x(t)\}$, $x^*$ denotes the upper bound of sequence, namely, $x^* = \sup_{t=1,2,\cdots} x(t)$.

\section{Distribution Free One-Pass Learning}
\label{section-3.1}
Since the sequence $\{\mathbf{w}(t)\}$ is changing over time in a dynamic environment, it is no longer reasonable to estimate current (i.e., at time $t$) concept via methods introduced previously. Instead, we introduce a sequence of discounted factors $\{\lambda(t)\}$ to downweight the loss of older data as follows,
\begin{equation}\label{dynamic-discount-targ}
\hat{\mathbf{w}}(t) = \mathop{\arg\min}_{\mathbf{w}\in \mathbb{R}^d}  \sum_{i=1}^t (\prod_{j=i+1}^t \lambda (j))\left[y(i) - \mathbf x(i)^\mathrm{T}\mathbf{w} \right]^2,
\end{equation}
where $\lambda(i) \in (0,1)$ is a discounted factor to smoothly put less weight on older data. The intuition can be more easily obtained if we simplify  all $\lambda(i)$ as a constant $\lambda \in (0,1)$, then the target function is,
\begin{equation}\label{dynamic-forgetting-targ}
\hat{\mathbf{w}}(t) = \mathop{\arg\min}_{\mathbf{w}\in \mathbb{R}^d} \ \sum_{i=1}^t \lambda^{t-i}\left[y(i) - \mathbf x(i)^\mathrm{T}\mathbf{w} \right]^2,
\end{equation}
And the quantity $\mu \triangleq 1-\lambda$ is named as \textit{forgetting factor}~\cite{book/Pearson/haykin2008adaptive}. The value of forgetting factor is, as a matter of fact, a trade-off between stability of past condition and sensitivity of future evolution. 

It should be pointed out that the forgetting mechanism based on exponential forgetting factor could be also considered as a continuous analogy to sliding window approach to some extent. The older data items with a small enough weight can be somehow thought as exclusion from the window. Some. More discussions on relation with window size and forgetting factor are provided in Section \ref{subsection-5.3-parameter}.

\subsection{Algorithm}
For the optimization problem proposed in~\eqref{dynamic-forgetting-targ}, obviously, by taking derivative of the target function, we can directly obtain the optimal solution in a closed-form,
\begin{equation}
\small
\label{forgetting-close-form}
[\mathbf{w}(t)]_{opt} = \left[\sum_{i=1}^t \lambda^{t-i} \mathbf x(i)\mathbf x(i)^\mathrm{T}  \right]^{-1} \left[\sum_{i=1}^t \lambda^{t-i} \mathbf x(i)y(i) \right]
\end{equation}

\begin{algorithm}[h]
   \caption{Distribution Free One-Pass Learning}
   \label{alg:DFOP}
\begin{algorithmic}
   \STATE {\bfseries Input:} A stream of data with $\{\mathbf{x}(t), \{y(t)\}_{t=1\cdots T}$, forgetting factor $\mu\in(0,1)$;
   \STATE {\bfseries Output:} Prediction $\{\hat{y}(t)\}_{t=1\cdots T}$ (real value for regression and discrete-value for classification).
   \STATE {Initialize $P_0>0$};
   \FOR{$t=1$ {\bfseries to} $T$}
   \STATE {$P(t) = \frac{1}{1-\mu}\left\lbrace P(t-1) - \mu \frac{P(t-1)\mathbf x(t)\mathbf x(t)^\mathrm{T}P(t-1)}{1-\mu+\mathbf x(t)^\mathrm{T}P(t-1)\mathbf x(t)}\right\rbrace$};  
   \STATE {$L(t) = P(t) \mathbf x(t)$};
   \STATE {$\hat{\mathbf{w}}(t) = \hat{\mathbf{w}}(t-1) +  \mu L(t)[y(t)-\hat{\mathbf{w}}(t-1)^\mathrm{T}\mathbf x(t)]$};
   \STATE {$\hat{y}(t) = \hat{\mathbf{w}}(t)^\mathrm{T}\mathbf x(t)$. \quad \quad \  // for regression};
   \STATE {$\hat{y}(t) = \textbf{sign}[\hat{\mathbf{w}}(t)^\mathrm{T}\mathbf x(t)].$ // for classification}
   \ENDFOR    
\end{algorithmic}
\end{algorithm}
However, above expression is an off-line estimate, namely, all the data items ahead of $t$ are needed. Instead of repeatedly solving \eqref{forgetting-close-form}, we estimate the underlying concept by adding a correction term to the previous estimate based on the information of new coming data item. With the forgetting factor recursive least square method~\cite{book/Pearson/haykin2008adaptive}, we could solve the target \eqref{dynamic-forgetting-targ} in a one-pass paradigm. And to the best of our knowledge, this is the very first time to adopt traditional forgetting factor RLS to deal with such tasks with distribution change under the one-pass constraints. And we named this as DFOP(short for Distribution-Free One-Pass) summarized in Algorithm~\ref{alg:DFOP}.

Besides, it should be pointed out that $\{\lambda(t)\}$ is by no means necessary chosen as a constant, we provide a generalized DFOP (short as G-DFOP) for a dynamic discount factor sequence $\{\lambda(t)\}$, corresponding to target in \eqref{dynamic-discount-targ}, which is also provably a one-pass algorithm. Detailed proofs are provided in Section 1 of supplementary material. 

For the classification scenario, $y(t)$ is no longer a real-valued output but a discrete value, and we assume $y(t)\in \{-1,+1\}$ for convenience.  A slight modification on original output step is applied in classification, where the effectiveness is empirically validated in the next section.

Assuming that the feature is $d$-dimension, we only need to keep $P(t)\in \mathbb{R}^{d\times d}$ in memory during the algorithm processing procedure. In other words, the storage is always $O(d^2)$, which is independent to the number of training examples. Besides, at the $t$-th time stamp, the update of $\hat{\mathbf{w}}(t)$ is unrelated to the previous data items, namely every data item can be discarded once it has been scanned.

\subsection{Theoretical Guarantee}
\label{section-theory}
In this section, we develop an estimate error bound in a non-stationary regression scenario. 

Consider the additive model of drift in sequence $\{\mathbf{w}(t)\}$,
\begin{equation}
\label{concept-drift}
\mathbf{w}(t) = \mathbf{w}(t-1) + \mathbf{s}(t), t\geq 1
\end{equation}
We assume that the adding term $\{\mathbf{s}(t)\}$ is an $E$-valued martingale-difference $d$-dimension sub-Gaussian vector sequence, with corresponding variance proxy sequence $\{\sigma_t\}$, whose formal definition will be given following. The $E$-valued martingale-difference assumption is reasonable, in fact, in many real-world application, the drift of concepts are usually independent.

Similar to the analysis in \cite{journals/control/guo1993performance}, we relax the assumptions to be more realistic in real-world applications and provide a non-deterministic estimate error bound $\lVert \mathbf{w}(t) - \hat{\mathbf{w}}(t) \rVert$ based on vector concentration, showing that the estimate error is tending to convergence with high probability.

Now we give the formal definitions of sub-Gaussian random variable and random vector.
\begin{defn} (sub-Gaussian random variable)
A random variable $X\in \mathbb{R}$ is said to be \textit{sub-Gaussian} with variance proxy $\sigma^2$ if $\mathbb{E}[X]=0$ and its moment generating function satisfies
\begin{equation}
\label{def:sub-gaussian}
\mathbb{E}[\exp(s X)] \leq \exp(\frac{s^2 \sigma^2}{2}),\quad	\forall s\in \mathbb{R}
\end{equation}
\end{defn}

\begin{defn} (sub-Gaussian random vector)
A random vector $\mathbf{x} \in \mathbb{R}^d = (x_1,\cdots,x_d)$ is called \textit{sub-Gaussian} with variance proxy $\sigma^2$ if all its coordinates are sub-Gaussian random variables with variance proxy $\sigma^2$.
\end{defn}

To exploit concentration property of sub-Gaussian random vector, condition ($\mathcal{C}_\alpha[\sigma^\infty]$) proposed in Theorem 2.1 of \cite{journal/juditsky2008large} shall be satisfied. Thus, first, we show that there exists a \textit{bounding sequence} $\{\gamma_t\}$ for a sub-Gaussian random vector sequence $\{\mathbf{x}(t)\}$.

\begin{lemma}
\label{theorem:sub-Gaussian}
For a sub-Gaussian random vector sequence $\{\mathbf{x}(t)\}$ with a variance proxy sequence $\{\sigma_t\}$, there exists a corresponding positive bounding sequence $\{\gamma_t\}$, such that 
\begin{equation}
\forall t\geq 1: \mathbb{E} \left\{ \exp\{\lVert \mathbf{x}(t)\rVert^2/\gamma_t^2\}\right\}\leq \exp\{1\}
\end{equation} 
\end{lemma}

Lemma \ref{theorem:sub-Gaussian} guarantees the "light tail" assumption of sub-Gaussian random vector. Then we could apply the following vector concentration, which is a corollary of Theorem 2.1 proposed in \cite{journal/juditsky2008large}.

\begin{thm}
\label{thm:concentration}
(Corollary of Theorem 2.1 in \cite{journal/juditsky2008large}) In an Euclidean space $(\mathbb{R}^n,\lVert \cdot \rVert_2)$, let E-valued martingale-difference sub-Gaussian sequence $\xi^\infty$ with a corresponding bounding sequence $\sigma^N = [\sigma_1;\cdots; \sigma_N]$. Let $S_N = \sum_{i=1}^N \xi_i$, then for all $N \geq 1$ and $\gamma \geq 0$:
\begin{eqnarray}
\label{concentration-condition}
\Pr \left\{ \lVert S_N \rVert \geq  \sqrt{2} (1+ \gamma)\sqrt{\sum_{i=1}^N \sigma_i^2} \right\} \leq  \exp\left\{ -\gamma^2/3\right\},
\end{eqnarray}
\end{thm}

Based on Theorem \ref{concentration-condition}, we could provide Lemma \ref{lemma:estimate-1} and Lemma \ref{lemma:estimate-2} to bound a sum of sub-Gaussian random vectors and random  variables with exponential decrease, respectively. 
\begin{lemma}
\label{lemma:estimate-1}
Let $\{\mathbf{x}(t)\}$ be an $E$-valued martingale-difference $d$-dimension sub-Gaussian random vector sequence, with corresponding bounding sequence $\{\gamma_t\}$, and $Z(t)\in \mathbb{R}^{d\times d}, Y(t) = (1-\mu)Y(t-1) + \mu Z(t), t\geq 1$. Then for $\mu \in (0,1)$, with a probability at least $1-\frac{\delta}{2}$, we have 
\begin{equation*}
\small
\left\lVert  \sum_{k=1}^t (1-\mu)^{t-k}Y(k)\mathbf{x}(k)\right\rVert \leq \sqrt{2}(1+\sqrt{3 \ln (2t/\delta)})\cdot \gamma^*(\|Y(0)\| +  Z^*) \mu^{-\frac{1}{2}}
\end{equation*}
where $Z^* = \sup_{k=1,\cdots,t}\lVert Z(k)\rVert$ and $\gamma^* = \sup_{k=1,\cdots,t}\gamma_k$.
\end{lemma}

\begin{lemma}
\label{lemma:estimate-2}
Let $\{\epsilon(t)\}$ be an independent (or $E$-valued martingale-difference) sub-Gaussian random variable sequence, with corresponding bounding sequence (i.e., variance proxy sequence) $\{\sigma_t\}$, and $\mathbf x(t)\in \mathbb{R}^{d}, t\geq 1$. Then for $\mu \in (0,1)$, with a probability at least $1-\frac{\delta}{2}$, we have 
\begin{equation*}
\small
\left\|  \sum_{k=1}^t (1-\mu)^{t-k}\mathbf{x}(k)\epsilon(k)\right\| \leq 2\sqrt{2}\left(1+\sqrt{3 \ln\frac{2t}{\delta}}\right)\sigma^* \mu^{-\frac{1}{2}}
\end{equation*}
where $\sigma^* = \sup_{k=1,\cdots,t}\lVert \mathbf{x}(k)\rVert \cdot \sup_{k=1,\cdots,t}\sigma_k $.
\end{lemma}

\begin{thm}
\label{thereom:error bound}
Assume following conditions be satisfied:\vspace{-3mm}
\begin{itemize}
\item[(\expandafter{\romannumeral1})] drift term $\{ \mathbf{s}(t)\}$ is an $E$-valued martingale-difference sub-Gaussian random vector sequence, with corresponding bounding sequence $\{\gamma_t\}$; \vspace{-1mm}
\item[(\expandafter{\romannumeral2})] output noise $\{\epsilon(t)\}$ is an independent (or $E$-valued martingale-difference) sub-Gaussian random variable sequence, with corresponding bounding sequence (i.e., variance proxy sequence) $\{\sigma_t\}$.\vspace{-1mm}
\end{itemize}
Then with a probability at least $1-\delta$, we have 
\begin{equation*}
\label{error-bound}
\begin{split}
\| \mathbf{w}(t)-\hat{\mathbf{w}}(t)\rVert &\leq K \left\{ \right.(1-\mu)^t\lVert R(0) \rVert \lVert \tilde{\mathbf{w}}(0) \rVert + \sqrt{2} (1+\sqrt{3 \ln (2t/\delta)}) \\
&\quad \left.\cdot [2\sigma^*\mu^{1/2} + \gamma^*(\|R(0)\|+{x^*}^2) \mu^{-1/2}]\right\}\\
\end{split}
\end{equation*}
where $K = \sup_{k=1,\cdots,t}\lVert P(k)\rVert$, $x^* = \sup_{k=1,\cdots,t}\lVert \mathbf{x}(k)\rVert$, $\sigma^* = \sup_{k=1,\cdots,t} \lVert \mathbf{x}(k)\rVert \cdot \sup_{k=1,\cdots,t}\sigma_k$ and $\gamma^* = \sup_{k=1,\cdots,t}\gamma_k$.
\end{thm}

\textbf{Remark.} The estimate error bound can be decomposed into three parts, i.e., the first one is $(1-\mu)^t\lVert R(0) \rVert \lVert \bm{\tilde{\mathbf{w}}}(0) \rVert$, second one is $2\sqrt{2} (1+\sqrt{3 \ln 2t/\delta})  \sigma^*\mu^{1/2}$ and third one is $\sqrt{2} (1+\sqrt{3 \ln 2t/\delta})(\|R(0)\|+{x^*}^2)\gamma^*\mu^{-1/2}$. Apparently, the first term is decreasing to zero as $t$ increases to infinity, second term is caused by the output noise which shall not be erased, and the third term is introduced by drift of $\mathbf{w}(t)$. Ignoring the poly-logarithmic
factors in $t$ and $d$, then, an asymptotic analysis gives the estimate error bound as,
\[
\lVert \mathbf{w}(t)-\hat{\mathbf{w}}(t)\rVert = \tilde{O}(\sqrt{\mu}+\frac{1}{\sqrt{\mu}})+o(1),\ w.h.p. 
\] 
where we use the $\tilde{O}$ notation to hide constant factors as well as poly-logarithmic factors in $t$ and $d$, and $o(1)$ will exponentially decrease to zero as $t \rightarrow \infty$.

Due to the page limits, we present the proofs of Theorem \ref{thm:concentration} and Theorem \ref{thereom:error bound} (along with Lemma \ref{theorem:sub-Gaussian}, \ref{lemma:estimate-1} and \ref{lemma:estimate-2}) in Section 2 and 3 of supplementary material, respectively.

\section{Experiments}
\label{section-4}
In this section, we examine the empirical performance of the proposed DFOP on both regression and classification scenarios. Then, we analyze the parameter sensitivity in Section \ref{subsection-5.3-parameter}. However, due to the page limits, only results on the classification scenario are provided, and the regression ones are appended in the supplementary materials.

Moreover, considering that when dealing with real-world datasets, we could not grasp the evolving distribution, specifically, the start and end time of drift, the underlying distribution. As a consequence, it would be very incomplete to analyze the behaviour of algorithms. Hence, both synthesis and real-world datasets are included in the comparison experiments. 

\subsection{Comparisons Methods}
\label{subsection-4.2}
We compare the proposed approach with six common methods on both synthesis and real-world datasets. The comparison methods are (a) RLS, least square approach solved in a recursive manner, (b) Sliding window approach, the classifier is constantly updated by the nearest data samples in the window. Base classifiers are 1NN and SVM, denoted as 1NN-win and SVM-win~\cite{conf/sdm/SouzaSGB15}, (c) SVM-fix, batch implementation of SVM with a fixed window size~\cite{conf/kdd/SyedLS99a}, (d) SVM-ada, ~batch implementation of SVM with an adaptive window size~\cite{journals/ida/Klinkenberg04}, (e) DWM, dynamic weighted majority algorithm, an adaptive ensemble based on the traditional weighted majority algorithm Winnow~\cite{conf/icdm/KolterM03,journals/jmlr/KolterM07}.

It's noteworthy to emphasize that the above comparisons are not all fair enough, because DFOP requires each data item be processed only once. Moreover, DFOP only needs one instance to update the model. Not all comparison methods can meets these two constraints, specifically, 1NN-win, SVM-win, SVM-fix and SVM-ada are window-based algorithms, hence, they are not one-pass. Besides, SVM-fix and SVM-ada are not incremental but updated in a series of batches. DWM is incremental style but not one-pass because it needs to use data to update experts pool in addition.

\subsection{Synthetic Datasets}
First, we present the performance comparisons over synthetic datasets.\vspace{-2mm}
\begin{itemize}
\item[-] \textit{SEA}~\cite{conf/kdd/StreetK01} consists of three attributes $x_1,x_2,x_3$, and $0.0\leq x_i \leq 10.0$. The target concept is $x_1+x_2 \leq b$, and there are 50,000 instances with 4 stages where $b \in \{7,8,9,9.5\}$. 
\item[-] \textit{hyperplane}~\cite{conf/kdd/Fan04}, is generated uniformly in a 10-dimensional hyperplane with 90,000 instances in total over 9 different stages.
\end{itemize}
\vspace{-2mm}
Besides, another 11 synthesis datasets for binary classification are also adopted. Detailed information are included in the supplementary materials.

The performance is measured by holdout accuracy since underlying joint distribution of synthetic datasets are known. Holdout accuracy is calculated over testing data generated according to the identical distribution as training data at each time stamp. Performance comparisons of seven approaches on SEA and hyperplane datasets are depicted in Figure~\ref{figure:Comparison}. Since the accuracy curves of SVM-ada, SVM-fix, 1NN-win and SVM-win are so unstable that they would shield all the other curves, we also present a relatively neat figure containing RLS, DWM and DFOP only.

\begin{figure}[b]
\centering
\includegraphics[width=\textwidth]{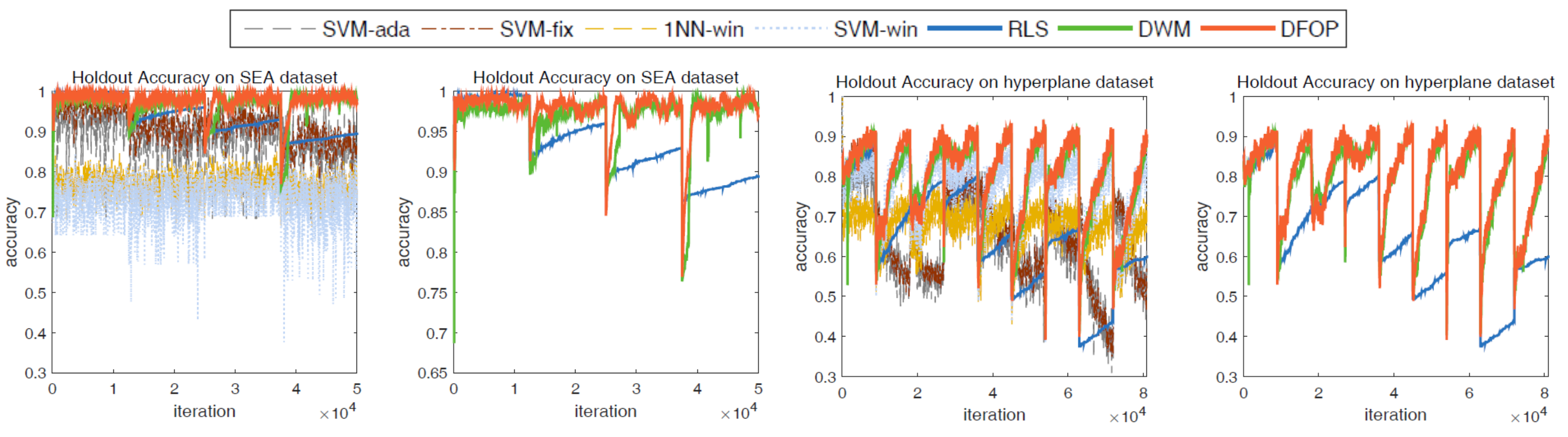}
\caption{\small{Performance comparison of seven approaches on synthetic datasets in terms of holdout accuracy. Left sides presents all the seven approaches, on the right side, only RLS, DWM and DFOP are plotted for clearness.}}
\label{figure:Comparison}
\end{figure}
As shown in Figure~\ref{figure:Comparison}, the accuracy of all algorithms falls rapidly when the underlying distribution emerges abrupt drift, and then will rise up with more data coming. DFOP is significantly better than RLS which is a special case of DFOP, this phenomenon validates the effectiveness of forgetting mechanism. Furthermore, the best two algorithms, obviously, are DFOP and DWM, both of them can converge to new stage quickly. DFOP shows a slightly better performance than DWM, both in slope and asymptote. Moreover, DWM requires to dynamically maintain a set of experts and needs previous data to update experts pool and to decide whether to remove poorly performing experts. On the contrary, DFOP demonstrates a desirable performance requiring to scan each data item only once.

\subsection{Real-world Datasets}
\begin{table*}[t]
\tiny 
\centering
\caption{\small{Performance comparison in terms of mean and standard deviation of accuracy (both in percents). Bold values indicates the best performance. Besides, $\bullet$ ($\circ$) indicates that DFOP is significantly better (worse) than the compared method (paired $t$-tests at 95\% significance level). And Win/ Tie/ Loss are summarized in the last row.}}\vspace{2mm}\label{table:Classification}
\begin{tabular}{l|llllll|l}\hline
Dataset   	 &\multicolumn{1}{c}{SVM-win}          & \multicolumn{1}{c}{1NN-win}         & \multicolumn{1}{c}{SVM-fix}          & \multicolumn{1}{c}{SVM-ada}          & \multicolumn{1}{c}{DWM }             & \multicolumn{1}{c}{RLS}              & \multicolumn{1}{c}{DFOP}           \\ \hline
SEA       	   & 73.94 $\pm$ 0.12$\bullet$ & 77.27 $\pm$ 0.04$\bullet$ 	& 86.19 $\pm$ 0.06$\bullet$ & 83.47 $\pm$ 0.09$\bullet$ & 87.04 $\pm$ 0.03$\bullet$ 	& 84.54 $\pm$ 0.47$\bullet$ & \textbf{87.99 $\pm$0.05}  \\
hyperplane	   & 83.74 $\pm$ 0.03$\bullet$ & 70.66 $\pm$ 0.03$\bullet$ 	& 87.98 $\pm$ 0.03$\bullet$ & 81.94 $\pm$ 0.07$\bullet$ & 88.36 $\pm$ 0.25$\bullet$ 	& 69.67 $\pm$ 1.40$\bullet$ & \textbf{90.14 $\pm$0.05}  \\
1CDT      	  & 98.71 $\pm$ 0.05$\bullet$ & 99.96 $\pm$ 0.07\ 			& 99.77 $\pm$ 0.06$\bullet$ & 99.77 $\pm$ 0.08$\bullet$ & 99.90 $\pm$ 0.09$\bullet$ 	& 98.79 $\pm$ 1.65$\bullet$ & \textbf{99.97 $\pm$0.05}  \\
2CDT      	  & 94.86 $\pm$ 0.06$\bullet$ & 94.62 $\pm$ 0.10$\bullet$ 	& 95.19 $\pm$ 0.13$\bullet$ & 95.18 $\pm$ 0.15$\bullet$ & 90.21 $\pm$ 0.67$\bullet$ 	& 62.24 $\pm$ 0.23$\bullet$ & \textbf{96.36 $\pm$0.09}  \\
1CHT      	  & 98.75 $\pm$ 0.18$\bullet$ & 99.81 $\pm$ 0.22\  			& 99.63 $\pm$ 0.17$\bullet$ & 99.63 $\pm$ 0.18$\bullet$ & 99.69 $\pm$ 0.26$\bullet$ 	& 98.49 $\pm$ 1.61$\bullet$ & \textbf{99.84 $\pm$0.16}  \\
2CHT      	  & 87.70 $\pm$ 0.04$\bullet$ & 85.69 $\pm$ 0.05$\bullet$ 	& 89.48 $\pm$ 0.12$\bullet$ & 88.89 $\pm$ 0.13$\bullet$ & 85.92 $\pm$ 0.72$\bullet$ 	& 62.57 $\pm$ 0.23$\bullet$ & \textbf{89.91 $\pm$0.07}  \\
1CSurr    	  & 97.99 $\pm$ 0.04$\bullet$ & \textbf{98.12 $\pm$ 0.11}$\bullet$ 	& 94.24 $\pm$ 1.08\ 		& 93.56 $\pm$ 1.08\ 		& 96.31 $\pm$ 0.50$\circ$ 		& 67.82 $\pm$ 0.22$\bullet$ & 93.24 $\pm$1.44  \\
UG-2C-2D      & 94.47 $\pm$ 0.13$\bullet$ & 93.55 $\pm$ 0.16$\bullet$ 	& 95.41 $\pm$ 0.10$\bullet$ & 94.92 $\pm$ 0.12$\bullet$ & 95.59 $\pm$ 0.11 				& 67.02 $\pm$ 1.46$\bullet$ & \textbf{95.59 $\pm$0.10}  \\
UG-2C-3D      & 93.60 $\pm$ 0.73$\bullet$ & 92.83 $\pm$ 0.93$\bullet$	& 95.05 $\pm$ 0.64$\bullet$ & 94.48 $\pm$ 0.71\ 		& 95.14 $\pm$ 0.62			 	& 61.95 $\pm$ 2.60$\bullet$ & \textbf{95.37 $\pm$0.61}  \\
UG-2C-5D      & 74.82 $\pm$ 0.45$\bullet$ & 88.04 $\pm$ 0.42$\bullet$ 	& 91.74 $\pm$ 0.26$\bullet$ & 90.37 $\pm$ 0.35$\bullet$ & 92.82 $\pm$ 0.23$\circ$ 		& 81.20 $\pm$ 2.42$\bullet$ & \textbf{92.51 $\pm$0.25}  \\
MG-2C-2D      & \textbf{90.20 $\pm$ 0.07}$\circ$   & 87.84 $\pm$ 0.09$\bullet$ 	& 84.98 $\pm$ 0.06$\bullet$ & 84.22 $\pm$ 0.06$\bullet$ & 90.15 $\pm$ 0.06$\circ$ 		& 57.18 $\pm$ 3.66$\bullet$ & 85.06 $\pm$0.06  \\
G-2C-2D 	  & 95.54 $\pm$ 0.01$\bullet$ & \textbf{99.61 $\pm$ 0.00}$\circ$		& 95.41 $\pm$ 0.01$\bullet$ & 95.26 $\pm$ 0.02$\bullet$ & 95.82 $\pm$ 0.02 				& 95.84 $\pm$ 0.01 			& 95.83 $\pm$0.02  \\ \hline
Chess         & 69.67 $\pm$ 1.51$\bullet$ & \textbf{79.58 $\pm$ 0.54} \ 			& 77.73 $\pm$ 1.56$\bullet$ & 69.18 $\pm$ 3.65$\bullet$ & 73.77 $\pm$ 0.66$\bullet$ 	& 78.70 $\pm$ 0.83$\bullet$ & 79.15 $\pm$0.62  \\
Usenet-1       & 68.92 $\pm$ 1.12\  		 & 65.36 $\pm$ 1.55$\bullet$ 	& 64.18 $\pm$ 2.24$\bullet$ & 67.68 $\pm$ 1.86$\bullet$ & 64.43 $\pm$ 4.53$\bullet$ 	& 60.65 $\pm$ 0.53$\bullet$ & \textbf{69.20 $\pm$0.68}  \\
Usenet-2       & 74.44 $\pm$ 0.71$\bullet$ & 71.03 $\pm$ 0.60$\bullet$ 	& 73.99 $\pm$ 0.69$\bullet$ & 72.64 $\pm$ 0.84$\bullet$ & 73.37 $\pm$ 0.93$\bullet$ 	& 73.16 $\pm$ 0.67$\bullet$ & \textbf{75.60 $\pm$0.57}  \\
Luxembourg     & 88.57 $\pm$ 0.28$\bullet$ & 77.51 $\pm$ 0.44$\bullet$ 	& 98.25 $\pm$ 0.19$\bullet$ & 97.43 $\pm$ 0.42$\bullet$ & 92.61 $\pm$ 0.40$\bullet$ 	& 99.06 $\pm$ 0.14$\bullet$	& \textbf{99.09 $\pm$0.14}  \\
Spam           & 83.91 $\pm$ 2.20$\bullet$ & 93.43 $\pm$ 0.82$\bullet$ 	& 92.44 $\pm$ 0.80$\bullet$ & 91.01 $\pm$ 0.94$\bullet$ & 91.49 $\pm$ 1.09$\bullet$ 	& 94.46 $\pm$ 0.16$\bullet$ & \textbf{94.77 $\pm$0.26}  \\
Weather        & 68.54 $\pm$ 0.55$\bullet$ & 72.64 $\pm$ 0.25$\bullet$ 	& 67.79  $\pm$ 0.65$\bullet$ & 77.26 $\pm$ 0.33$\bullet$ & 70.86 $\pm$ 0.42$\bullet$ 	& 78.35 $\pm$ 0.18$\bullet$ & \textbf{79.23  $\pm$0.12}  \\
Powersupply    & 73.33 $\pm$ 0.25$\bullet$ & 72.42 $\pm$ 0.21$\bullet$ 	& 71.17 $\pm$ 0.15$\bullet$ & 69.39 $\pm$ 0.17$\bullet$ & 72.18 $\pm$ 0.29$\bullet$ 	& 69.67 $\pm$ 0.64$\bullet$ & \textbf{80.46 $\pm$0.04}  \\
Electricity    & 74.20 $\pm$ 0.08$\bullet$ & \textbf{85.33 $\pm$ 0.09}$\circ$ 		& 62.01 $\pm$ 0.59$\bullet$ & 58.69 $\pm$ 0.58$\bullet$ & 78.60 $\pm$ 0.41$\bullet$ 	& 74.20 $\pm$ 0.63$\bullet$ & 76.94 $\pm$0.26  \\\hline
DFOP \tiny{W/ T/ L} &\multicolumn{1}{c}{18/ 1/ 1} & \multicolumn{1}{c}{14/ 4/ 2}	&\multicolumn{1}{c}{19/ 1/ 0}	& \multicolumn{1}{c}{18/ 2/ 0}	& \multicolumn{1}{c}{14/ 3/ 3}	&\multicolumn{1}{c}{19/ 1/ 0}	&	\multicolumn{1}{c}{-}\\ \hline
\end{tabular}
\end{table*}

To valid the effectiveness of DFOP in real-world applications, performance comparisons are presented over 8 real-world datasets. Detailed descriptions are provided in Section 5 of supplementary materials. 

In real-world datasets, we can never expect to foreknow the underlying distribution at each data stamp. Thus, it's not possible to still adopt holdout accuracy as performance measurement. In Table \ref{table:Classification}, we conduct all the experiments for 10 trails and report the overall mean and standard deviation of predictive accuracy over above real-world datasets as well as other 12 synthesis datasets.

In a total of 20 datasets, the number of instance vary from 533 to at most 200,000. DFOP achieves the best among all approaches in 15 over 20 datasets. Also, in other 5 datasets, DFOP ranks the second or the third. This validates the effectiveness of DFOP, especially under an unfair comparison condition. 

Additionally, the robustness\cite{conf/kdd/VlachosDGKK02} of all these different algorithms are compared. Briefly speaking, for a particular algorithm \textit{algo}, the robustness $r_{algo}$ is defined as the proportion between its accuracy and the smallest accuracy among all compared algorithms, i.e., $r_{algo} = acc_{algo}/\min_{\alpha} acc_\alpha$. Hence, the sum of $r_{algo}$ over all datasets indicates the robustness of for algorithm $algo$. The greater the value of the sum, the better the performance. DFOP achieves the best over 20 datasets, and RLS ranks last as expected since it didn't consider the evolving distribution in datasets at all. Due to the page limits, detailed robustness comparison results could be found in Section 4.3 in supplementary materials. 

\subsection{Parameter Study}
\label{subsection-5.3-parameter}
As stated previously, how to choose an appropriate forgetting factor is an important issue since it reflects a trade-off between stability of past condition and sensitivity to future evolution. To figure out how forgetting factor affects the performance, in classification problem, accumulated accuracy (short as 'AA') is adopted as a performance measurement in the time series and is defined as,\vspace{-2mm}
\begin{equation}
\text{AA}(t) = \sum_{i=1}^t \mathbb{I} [\hat{y}(i)=y(i)]/t,
\end{equation}
where $\mathbb{I}(\cdot)$ is indicator function which takes 1 if $\cdot$ is true, and 0 otherwise, $\hat{y}(i)$ and $y(i)$ are predictive and ground-truth label, respectively. Figure \ref{figure:parameter} shows the impact of different forgetting factor $\mu$ over four datasets. We notice that the accumulated accuracy of RLS almost decreases all the time. For a relatively small but not zero $\mu$, the performance is satisfying without a significant gap. However, when $\mu$ is too large, say 0.5, the performance is even much worse than RLS. This is consistent with intuition since forgetting factor is so large and older data samples are exponentially downweighted that there are not sufficient effective training samples available to update the model. 
\begin{figure}[t]
\centering
\includegraphics[width=1\textwidth]{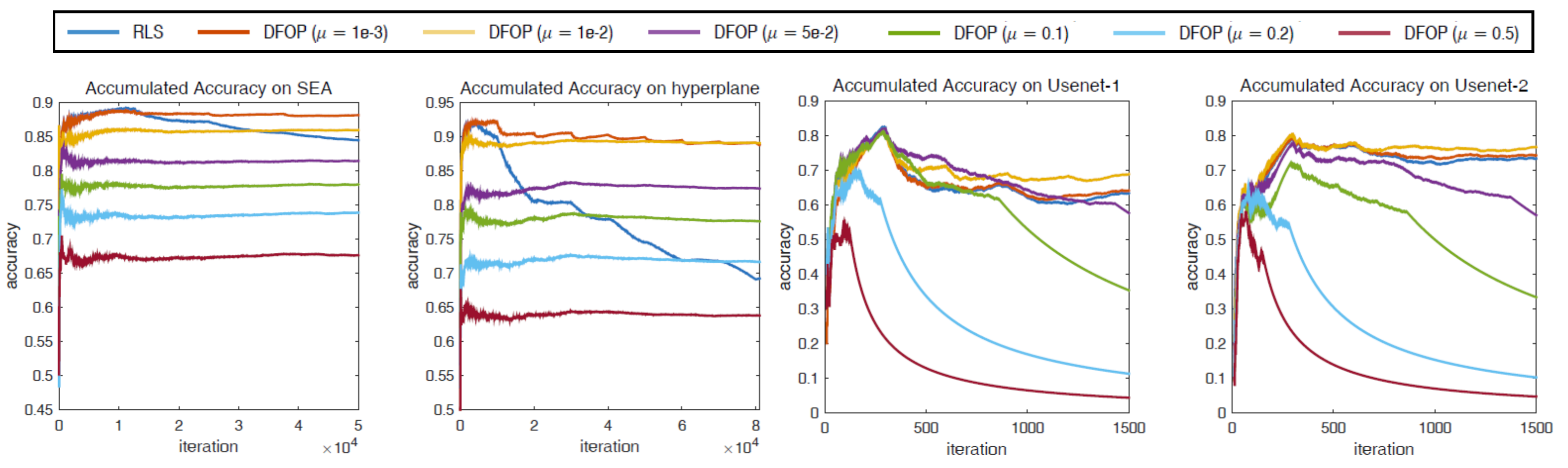}
\vspace{-5mm}
\caption{\small{Accumulated accuracy with different forgetting factors over four datasets with distribution change.}}
\label{figure:parameter}
\end{figure}
Now, here comes the question: how to choose an appropriate forgetting factor to adapt the distribution change in the data stream? To answer this question, let's recall the target function in \eqref{dynamic-forgetting-targ}, when $\mu$ is close to $0$, which is often the case in practice and validated in Figure \ref{figure:parameter}, then we have
\begin{equation}
 (1-\mu)^t = e^{t\ln(1-\mu)}\approx e^{-t\mu} = e^{-t/T_0}, 
\end{equation}
where we define $T_0 = 1/\mu$ as \textit{forgetting period}. The contribution for prediction error of data items older than $T_0$ time will be discounted with a weight less than  $e^{-1}\approx 36.8\%$ comparing to the current data.  As a matter of fact, the forgetting period in forgetting mechanism is pretty similar to the window size in sliding window technique. It can be regarded as a soft relaxation of window size. Consequently, the forgetting factor $\mu$ shall be chosen according to the forgetting period $T_0$, where the data distribution should be relatively smooth and stable during this forgetting period. 

We validate this idea over synthesis datasets reported in Table \ref{table:Parameter}. Theoretical recommended value $1/T_0$ and empirical appropriate value $\mu$ for forgetting factor are provides. Also, the relative proportions between them are calculated. We can see that these two value are very close over all datasets with no more than 20 times difference, even no more than 5 times in most datasets. This supports our strategy in choosing forgetting factor. 

\begin{table}[t]
\centering
\small
\caption{\small{Datasets, the number of data items between consecutive distribution change, theoretical recommended value and empirical appropriate value for forgetting factor are listed below, and the last column provides relative proportion between theoretical recommended value and empirical appropriate value.}}\vspace{2mm}
\label{table:Parameter}
\begin{tabular}{l|rccc||l|rccc}
\hline
Dataset     & \multicolumn{1}{c}{$T_0$}  & $1/T_0$  & $\mu$   & $\mu T_0$ 	&Dataset     & \multicolumn{1}{c}{$T_0$}  & $1/T_0$  & $\mu$   & $\mu T_0$\\ \hline
1CDT        & 400    & 2.50E-03 & 1.00E-02 & 4  	& UG-2C-2D    & 1,000  & 1.00E-03 & 1.00E-03 & 1  	\\
2CDT        & 400    & 2.50E-03 & 1.00E-02 & 4  	& UG-2C-3D    & 2,000  & 5.00E-04 & 1.00E-03 & 2  	\\
1CHT        & 400    & 2.50E-03 & 1.00E-02 & 4  	& UG-2C-5D    & 2,000  & 5.00E-04 & 1.00E-03 & 2  	\\
2CHT        & 400    & 2.50E-03 & 1.00E-02 & 4  	& MG-2C-2D    & 2,000  & 5.00E-04 & 1.00E-03 & 2  	\\
1CSurr      & 600    & 1.67E-03 & 5.00E-03 & 3  	& G-2C-2D & 2,000  & 5.00E-04 & 1.00E-03 & 2  	\\
hyperplane  & 9,000  & 1.11E-04 & 2.00E-03 & 18  	& SEA         & 10,000 & 1.00E-04 & 1.00E-03 & 10 	\\ \hline
\end{tabular}
\end{table}

Certainly, the drifting properties of real-world datasets are not as clear as synthetic datasets. Nevertheless, we could still infer the forgetting period $T_0$ based on the domain knowledge and choose an appropriate value as forgetting factor. For instance, considering the weather forecast dataset, although we cannot foreknow the drifting property of distribution, a relative stable period can still be estimated. 

\section{Conclusion}
In this paper, we proposed an approach based on forgetting mechanism called DFOP handling streaming learning problems with distribution change. The main idea is to downweight the older data items by introducing exponential forgetting factor without considering any prior about drifting information. Meanwhile, DFOP meets the one-pass constraints guaranteeing that only once will the data items be scanned without storing the entire dataset. Hence, DFOP. The storage requirement of DFOP is $O(d^2)$, where $d$ is the dimension of data, independent from the number of training examples. Both theoretical supports and empirical demonstrations for DFOP are presented to validate its effectiveness and practicality. 

Besides, how to efficiently reduce the storage and make DFOP paralleled to adapt a even larger scale real-world applications would be an interesting future work.

\section*{Appendix}
\setcounter{section}{0}
\section{Generalized DFOP and Proofs}
\subsection{Recursive Algorithm for Dynamic Discounted Factors}
\label{subsection: Appendix-1}
When dynamic discounted factors are introduced to downweight the contribution of older data items, the target function can be written as,
\begin{equation}
\label{dynamic-discount-targ}
\hat{\mathbf{w}}(t) = \mathop{\arg\min}_{\mathbf{w}\in \mathbb{R}^d}  \sum_{i=1}^t (\prod_{j=i+1}^t \lambda (j))\left[y(i) - \mathbf x(i)^\mathrm{T}\mathbf{w} \right]^2
\end{equation}
In this part, we provide a provably recursive algorithm to directly solve \eqref{dynamic-discount-targ} as shown in Algorithm \ref{alg:Downweight} named generalized DFOP algorithm, short as G-DFOP. 
\begin{algorithm}[H]
   \caption{Generalized DFOP}
   \label{alg:Downweight}
\begin{algorithmic}
   \STATE {\bfseries Input:} A stream of data with feature~$\{\mathbf{x}(t)\}_{t=1\cdots T}$~ and ~$\{y(t)\}_{t=1\cdots T}$, discounted factor sequence $\{\lambda(t)\}_{t=1\cdots T}$;
   \STATE {\bfseries Output:} Prediction label $\{\hat{y}(t)\}_{t=1\cdots T}$ (real value for regression and discrete-value for classification).
   \STATE {Initialize $P_0>0$},
   \FOR{$t=1$ {\bfseries to} $T$}
   \STATE {$P(t) = \frac{1}{\lambda(t)}[ P(t-1) - \frac{P(t-1)\mathbf{x}(t)\mathbf{x}(t)^\mathrm{T}P(t-1)}{\lambda(t) + \mathbf{x}(t)^\mathrm{T} P(t-1)\mathbf{x}(t)}]$}, 
   \STATE {$L(t) = P(t) \mathbf x(t)$},
   \STATE {$\hat{\mathbf{w}}(t) = \hat{\mathbf{w}}(t-1) + L(t)[y(t)-\hat{\mathbf{w}}(t-1)^\mathrm{T}\mathbf x(t)]$}.
   \STATE {$\hat{y}(t) = \hat{\mathbf{w}}(t)^\mathrm{T}\mathbf x(t)$. \quad \quad \  // for regression};
   \STATE {$\hat{y}(t) = \textbf{sign}[\hat{\mathbf{w}}(t)^\mathrm{T}\mathbf x(t)].$ // for classification}
   \ENDFOR    
\end{algorithmic}
\end{algorithm}

\subsection{Proof of Generalized DFOP}
In this part, we will prove the consistency between G-DFOP and target function in \eqref{dynamic-discount-targ}.

\begin{lemma}
\label{lemma-matrix-inverse}
Let $A,B,C$ and $D$ be matrices of compatible dimensions such that the product $BCD$ and the sum $A + BCD$ exist. Then we have 
\begin{equation}
\label{lemma-product}
[A+BCD]^{-1} = A^{-1} - A^{-1}B[DA^{-1}B + C^{-1}]^{-1}DA^{-1} 
\end{equation}
\end{lemma}
\begin{proof}
Multiply the right-hand side of \eqref{lemma-product} by $A+BCD$ from the right, this gives
\begin{equation*}
\begin{split}
& \left\{ A^{-1} - A^{-1}B[DA^{-1}B + C^{-1}]^{-1}DA^{-1}\right\}[A+BCD] \\ 
=\  & I + A^{-1}BCD -A^{-1}B[DA^{-1}B+C^{-1}]^{-1}D\\
& - A^{-1}B[DA^{-1}B+C^{-1}]^{-1}DA^{-1}BCD\\
=\ & I + A^{-1}B[DA^{-1}B+C^{-1}]^{-1}\{0\} = I
\end{split}
\end{equation*}
\end{proof}
For convenience, let $\Lambda(i,t)=\prod_{j=i+1}^t \lambda (j)$, then the close-form solution of optimization \eqref{dynamic-discount-targ} can be calculated as follows
\begin{equation}
\small
\label{close-form}
\hat{\mathbf{w}}(t) = \left[\sum_{i=1}^t \Lambda(i,t) \mathbf x(i)\mathbf x(i)^\mathrm{T}  \right]^{-1} \left[\sum_{i=1}^t \Lambda(i,t) \mathbf x(i)y(i) \right]
\end{equation}
Now we will prove that the solution obtained by G-DFOP is equivalent to close-form solution in \eqref{close-form}.
~\\

\textbf{Theorem.}\ \textit{By the policy in Algorithm \ref{alg:Downweight}, we can achieve the same solution as result in \eqref{close-form}.}

\proof Denote $\bar{R}(t) = \sum_{k=1}^t \Lambda(i,t) \mathbf x(t)\mathbf x(t)^\mathrm{T}$, obviously
\[
\bar{R}(t) = \lambda(t)\bar{R}(t-1) + \mathbf x(t)\mathbf x(t)^\mathrm{T}
\]
Then the solution in \eqref{close-form} can be rewritten into the following form:
\begin{equation*}
\label{proof-close-form}
\begin{split}
\hat{\mathbf{w}}(t) &= \bar{R}^{-1}(t)\left[\lambda(t) \sum_{i=1}^{t-1} \Lambda(i,t-1) \mathbf x(i)y(i) + \mathbf x(t)y(t)\right]\\
&= \bar{R}^{-1}(t)\left[ \lambda(t)\bar{R}(t-1)\hat{\mathbf{w}}(t-1) + \mathbf x(t)y(t)\right]\\
&= \bar{R}^{-1}(t)\left[ (\bar{R}(t) - \mathbf x(t)\mathbf x(t)^\mathrm{T}) \hat{\mathbf{w}}(t-1) + \mathbf x(t)y(t)\right]\\
&= \hat{\mathbf{w}}(t-1) + \bar{R}^{-1}(t)\mathbf x(t)\left[y(t) - \mathbf x(t)^\mathrm{T} \hat{\mathbf{w}}(t-1))\right]
\end{split}
\end{equation*}

Now, we introduce  $P(t) =  \bar{R}^{-1}(t)$ and then apply Lemma \ref{lemma-matrix-inverse} to \eqref{proof-close-form}, this gives
\begin{equation*}
\begin{split}
P(t) & =  \frac{1}{\lambda(t)}P(t-1)- \frac{1}{\lambda(t)}P(t-1)\mathbf{x}(t) \\
& \cdot \left[\frac{1}{\lambda(t)}\mathbf{x}(t)^\mathrm{T}P(t-1)\mathbf{x}(t) +1\right]^{-1} \frac{1}{\lambda(t)}\mathbf{x}(t)^\mathrm{T}P(t-1)\\
&= \frac{1}{\lambda(t)}\left[ P(t-1) - \frac{P(t-1)\mathbf{x}(t)\mathbf{x}(t)^\mathrm{T}P(t-1)}{\lambda(t) + \mathbf{x}(t)^\mathrm{T} P(t-1)\mathbf{x}(t)}\right]\\
\end{split}
\end{equation*}
Let $L(t) = P(t)\mathbf{x}(t)$, we can obtain the policy described in Algorithm \ref{alg:Downweight}.

\textbf{Remark.}  Obviously, DFOP in paper is only a special case when fixing discounted factor sequence $\{\lambda(t)\}$ as $(1-\mu)$. Note that, for a simplicity notations in estimate error analysis of DFOP, we slightly modified $L(t)$ and $P(t)$ multiplying by $\mu$.

\section{Proof of Theorem 1}

\subsection{Definition of Sub-Gaussian and Sub-Exponential}
First we give typical definitions of sub-Gaussian random variable and random vector, meanwhile, definition of sub-Exponential random variable is also provided.
\begin{defn} (sub-Gaussian random variable) A random variable $X\in \mathbb{R}$ is said to be \textit{sub-Gaussian} with variance proxy $\sigma_2$ if $\mathbb{E}[X]=0$ and its moment generating function satisfies
\begin{equation}
\label{def:sub-gaussian}
\mathbb{E}[\exp(s X)] \leq \exp(\frac{s^2 \sigma^2}{2}),\quad	\forall s\in \mathbb{R}
\end{equation}
In this case we write $X \sim \texttt{subG}(\sigma^2)$.
\end{defn}

\begin{defn} (sub-Gaussian random vector) A random vector $\mathbf{x} \in \mathbb{R}^d = (x_1,\cdots,x_d)$ is called \textit{sub-Gaussian} with variance proxy $\sigma^2$ if all its coordinates are sub-Gaussian random variables with variance proxy $\sigma^2$.
\end{defn}

\begin{defn} (sub-Exponential random variable) A random variable $X\in \mathbb{R}$ is said to be \textit{sub-Exponential} with parameter $\lambda$ if $\mathbb{E}[X]=0$ and its moment generating function satisfies
\begin{equation}
\label{def:sub-exponential}
\mathbb{E}[\exp(s X)] \leq \exp(\frac{s^2 \lambda^2}{2}),\quad	\forall |s|\leq \frac{1}{\lambda}.
\end{equation}
In this case we write $X \sim \texttt{subE}(\lambda)$.
\end{defn}

\textbf{Remark.  }
Attention that definitions above all require a zero-mean constraint, which is not necessary in analysis "light tail" property. Hence, for random variable that is not zero-mean but satisfies condition \eqref{def:sub-gaussian} is called generalized sub-Gaussian. And definitions for generalized sub-Gaussian random vector and generalized sub-Exponential random variable are similar. 
 
\subsection{Proof of Theorem 1}
Theorem 1 in the paper presents a vector concentration inequality is shown for sub-Gaussian random vector sequence in the following which plays an important role in proving Lemma 2 and Lemma 3 in our paper. To prove Theorem 1 in the paper, first, we present following Lemma \ref{lemma:random vector} to show that the norm of a sub-Gaussian random vector is a generalized sub-Gaussian random variable.
\begin{lemma}
\label{lemma:random vector}
If $\mathbf{x}\in \mathbb{R}^d$ is a sub-Gaussian random vector with variance proxy $\sigma^2$, then 
\begin{equation}
\mathbb{E}\exp(\lVert \mathbf{x} \rVert) \leq 2\exp\left\{\frac{d^2\sigma^2}{2}\right\}
\end{equation}
\end{lemma}
\proof The conclusion here is a direct corollary from  Theorem 3.1 in \cite{journals/buldygin2010inequalities}, where $\mathbf{c} = \mathbf{1} = (1,\cdots,1)$, $G = \lVert \cdot \rVert$, and $B(\mathbf{c},\mathbf{x})=\sum_{k=1}^d c_k \tau(x_k) = d\sigma$. 

Then we present the following Lemma \ref{lemma:equivalent} to show the equivalence between sub-Gaussian random variable and sub-Exponential variable, which plays an important role in proving Theorem 1.
\begin{lemma}
\label{lemma:equivalent}
Let $X$ be a sub-Gaussian random variable, i.e., $X\sim \texttt{subG}(\sigma^2)$. Then the random variable $Z = X^2 -\mathbb{E}[X^2]$ is sub-Exponential: $Z\sim \texttt{subE}(16\sigma^2)$. 
\end{lemma}
\proof We prove this lemma by definition.
\begin{eqnarray}
\mathbb{E}[e^{sZ}] &=& 1+ \sum_{k=2}^\infty \frac{s^k \mathbb{E}[X^2-\mathbb{E}[X^2]]^k}{k!} \nonumber\\
\label{eq-jensen-1}
& \leq & 1+ \sum_{k=2}^\infty \frac{s^k 2^{k}( \mathbb{E}[X^{2k}-(\mathbb{E}[X^2])^k)}{k!}\\ 
\label{eq-jensen-2}
& \leq & 1+ \sum_{k=2}^\infty \frac{s^k 4^{k}( \mathbb{E}[X^{2k}]}{2(k!)}\\
\label{eq-lemma4}
& \leq & 1+ \sum_{k=2}^\infty \frac{s^k 4^{k}2(2\sigma^2 k)}{2(k!)}\\
& = & 1+ (8s\sigma^2)^2 \sum_{k=0}^\infty (8s\sigma^2)^k \nonumber \\
\label{eq-def}
& = & 1 + 128s^2\sigma^4 \leq e^{128s^2\sigma^4}
\end{eqnarray}
where \eqref{eq-jensen-1} and \eqref{eq-jensen-2} because of Jensen's Inequality. \eqref{eq-lemma4} holds because of $\mathbb{E}[\lVert X\rVert^k]\leq (2\sigma^2)^{k/2}k\Gamma(k/2)$ in \cite{journal/vershynin2010intro}. The last step \eqref{eq-def} holds because the condition in definition of sub-Exponential, i.e., $|s| \leq 1/(16\sigma^2)$.

\textbf{Remark.}  Attention that in the proof of Lemma \ref{lemma:equivalent}, we didn't use the zero mean property of sub-Gaussian random variable $X$. Hence, Lemma \ref{lemma:equivalent} can be applied to a generalized sub-Gaussian random variable without requiring zero-mean condition.

Now, let's begin prove the Lemma 2 in the paper stated as follows,

\textit{For a sub-Gaussian random vector sequence $\{\mathbf{x}(t)\}$ with a variance proxy sequence $\{\sigma_t\}$, there exists a corresponding positive bounding sequence $\{\gamma_t\}$, such that 
\begin{equation}
\forall t\geq 1: \mathbb{E} \left\{ \exp\{\lVert \mathbf{x}(t)\rVert^2/\gamma_t^2\}\right\}\leq \exp\{1\}
\end{equation}}

\proof Consider any vector in the sequence, say $\mathbf{x}(t)$. Because it is a sub-Gaussian random vector, directly applying Lemma \ref{lemma:random vector}, we have 
\begin{equation*}
\mathbb{E}\exp(\lVert \mathbf{x}(t) \rVert) \leq 2\exp\left\{\frac{d^2\sigma_t^2}{2}\right\}
\end{equation*}
which means for a sub-Gaussian random vector, its norm is a generalized sub-Gaussian random variable. Here, "generalized" means it may not meet the zero-mean condition.

Because we didn't use the zero mean property of sub-Gaussian random variable in the proof of Lemma \ref{lemma:equivalent}, it can also be applied to generalized sub-Gaussian random variable. Let  $Z = \lVert \mathbf{x}(t) \rVert^2 - \mathbb{E}[\lVert \mathbf{x}(t) \rVert^2]$, then $Z\sim \texttt{subE}(16\sigma_t^2)$, specifically,
\[
\mathbb{E}[\exp\{ s(\lVert \mathbf{x}(t) \rVert^2 - \mathbb{E}[\lVert \mathbf{x}(t) \rVert^2])\}] \leq \exp\{s^2 16\sigma_t^2/2\}, \quad \forall |s| \leq \frac{1}{4\sigma_t}
\]
Thus, for $\forall |s| \leq 1/(4\sigma_t)$, we have 
\begin{equation}
\label{eq-condition}
\mathbb{E}[\exp\{ s\lVert \mathbf{x}(t) \rVert^2\}] \leq \exp\{8s^2 \sigma_t^2+ s\mathbb{E}[\lVert \mathbf{x}(t) \rVert^2]\}, \quad 
\end{equation}
Obviously, we can choose a sufficient small positive constant $s$ as $\gamma_t$, such that
\[
\mathbb{E} \left\{ \exp\{\lVert \mathbf{x}(t)\rVert^2/\gamma_t^2\}\right\}\leq \exp\{1\}.
\]
And $\{\gamma_t\}$ is exactly the bounding sequence as we desired. This completes the proof.

Now, we prove the Theorem 1 in the paper stated as follows,
\begin{thm}
\label{lemma:concentration}
In an Euclidean space $(\mathbb{R}^n,\lVert \cdot \rVert_2)$, let E-valued martingale-difference sub-Gaussian sequence $\xi^\infty$ with a corresponding bounding sequence $\sigma^N = [\sigma_1;\cdots; \sigma_N]$. Let $S_N = \sum_{i=1}^N \xi_i$, then for all $N \geq 1$ and $\gamma \geq 0$:
\begin{eqnarray}
\label{concentration-condition}
\Pr \left\{ \lVert S_N \rVert \geq  \sqrt{2} (1+ \gamma)\sqrt{\sum_{i=1}^N \sigma_i^2} \right\} \leq  \exp\left\{ -\gamma^2/3\right\},
\end{eqnarray}
\end{thm}

\proof First of all, for a sub-Gaussian sequence $\xi^\infty$ with corresponding bounding sequence $\sigma^\infty$, from Theorem 1 in paper, we have $\mathbb{E}_{i-1}\left\{ \exp\{\frac{\lVert \xi_i \rVert^2}{\sigma_i^2}\} \right\} = \exp\{1\}$. Besides, an Euclidean space $(\mathbb{R}^n,\lVert \cdot \rVert_2)$ is 1-smooth and 1-regular, which means $\kappa =1$. Then, it follows immediately according to Theorem 2.1 proposed in \cite{journal/juditsky2008large}.

\section{Proof of Theorem 2}
Firstly, the following generalized summation by parts is essential for the proof of main theorem in our paper.

\begin{lemma}
\label{lemma:summation}
Let $\left \langle {f_n} \right \rangle$ and $\left \langle {g_n} \right \rangle$ be two sequences. And denote $G_n = \sum_{k = 1}^n g_k$, then for $m\geq 2$ we have,
\begin{equation*}
\small
\begin{split}
\sum_{k=m}^n f_k g_k &= f_n G_n - f_m G_{m-1} - \sum_{k = m}^{n - 1} G_k\left( f_{k+1} - f_k \right)\\
\end{split}
\end{equation*}
\end{lemma}

\proof The right-hand side can be easily verified by expanding $g_k$ in the left-hand as $ G_k - G_{k-1}$.

\textbf{Remark.}  When $m=1$, the same derivation gives the famous Abel transformation,
\[
\sum_{k = 1}^n f_k g_k = f_n G_n - \sum_{k = 1}^{n - 1} G_k\left( f_{k+1} - f_k \right)
\]

\subsection{Proof of Lemma 2 in paper}
Lemma 2 in our paper states as follows,

\textit{Let $\{\mathbf{x}(t)\}$ be an $E$-valued martingale-difference $d$-dimension sub-Gaussian random vector sequence, with corresponding bounding sequence $\{\gamma_t\}$, and $Z(t)\in \mathbb{R}^{d\times d}, Y(t) = (1-\mu)Y(t-1) + \mu Z(t), t\geq 1$. Then for $\mu \in (0,1)$, with a probability at least $1-\frac{\delta}{2}$, we have 
\begin{equation*}
\small
\begin{split}
\left\lVert  \sum_{k=1}^t (1-\mu)^{t-k}Y(k)\mathbf{x}(k)\right\rVert \leq \sqrt{2}(1+\sqrt{3 \ln (2t/\delta)}) \cdot \gamma^*(\|Y(0)\| +  Z^*) \mu^{-\frac{1}{2}}
\end{split}
\end{equation*}
where $Z^* = \sup_{k=1,\cdots,t}\lVert Z(k)\rVert$ and $\gamma^* = \sup_{k=1,\cdots,t}\gamma_k$.}

\begin{proof}
Denote $S(t,k)\triangleq \sum_{i=k}^t \mathbf{x}(i)$, then with the recursive property of $Y(t)$, left side can be expanded by Lemma \ref{lemma:summation},
\begin{eqnarray*}
\left\|  \sum_{k=1}^t (1-\mu)^{t-k}Y(k)\mathbf{x}(k)\right\| &=&  \left\lVert  (1-\mu)^tY(0)S(t,1) + \sum_{k=1}^t \mu (1-\mu)^{k-1}Z(t-k+1)S(t,t-k+1) \right\rVert \\
\label{eq-Minkowski}
&\leq & \ (1-\mu)^t\|Y(0)S(t,1)\| + \sum_{k=1}^t \mu (1-\mu)^{k-1}\|Z(t-k+1)S(t,t-k+1)\| \\
\label{eq-cauchy}
&\leq & \ (1-\mu)^t\|Y(0)\| \left\lVert\sum_{i=1}^t \mathbf{x}(i)\right\rVert  \nonumber + Z^*\sum_{k=1}^t \mu (1-\mu)^{k-1} \left\lVert \sum_{i=t-k+1}^t \mathbf{x}(i)\right\rVert
\end{eqnarray*}
And it could be separated into two parts, i.e.,
\[
(a) = \ (1-\mu)^t\|Y(0)\| \left\lVert\sum_{i=1}^t \mathbf{x}(i)\right\rVert; \text{ and }(b) = Z^*\sum_{k=1}^t \mu (1-\mu)^{k-1} \left\lVert \sum_{i=t-k+1}^t \mathbf{x}(i)\right\rVert.
\] 

For $k= 1,\cdots, t$, based on Lemma~\ref{lemma:concentration}, we have
\begin{eqnarray*}
\small
\Pr \left\{ \left\lVert \sum_{i=k}^t \mathbf{x}(i)\right\rVert \leq \sqrt{2}(1+\sqrt{3 \ln\frac{2t}{\delta}})\sqrt{\sum_{i=k}^t \gamma_i^2} \right\}\geq  1-\frac{\delta}{2t}
\end{eqnarray*}
taking by the union bound over $t>0$, and let $\gamma^* = \sup_{k=1,\cdots,t}\gamma_k$ the following holds in a probability at least $1-\frac{\delta}{2}$,
\begin{eqnarray*}
(b) &\leq & \sqrt{2}\left(1+\sqrt{3 \ln\frac{2t}{\delta}}\right) Z^*\sum_{k=1}^t \mu (1-\mu)^{k-1} \sqrt{\sum_{i=t-k+1}^t \gamma_i^2}\\
&\leq & \sqrt{2}\left(1+\sqrt{3 \ln\frac{2t}{\delta}}\right) Z^* \gamma^*\sum_{k=1}^t \mu (1-\mu)^{k-1} \sqrt{k}\\
&\leq & \sqrt{2}\left(1+\sqrt{3 \ln\frac{2t}{\delta}}\right) Z^* \gamma^* \mu^{-1/2} \\
\end{eqnarray*}
Conditioning on all above concentration inequalities hold, for $(a)$, we have 
\begin{eqnarray*}
(a) &\leq & (1-\mu)^t\|Y(0)\| \sqrt{2}(1+\sqrt{3 \ln\frac{2t}{\delta}})\sqrt{\sum_{i=k}^t \gamma_i^2}\\
& \leq & \sqrt{2}(1+\sqrt{3 \ln\frac{2t}{\delta}})\|Y(0)\| \gamma^*(1-\mu)^t \sqrt{t} \\
& \leq & \sqrt{2}(1+\sqrt{3 \ln\frac{2t}{\delta}})\|Y(0)\| \gamma^* \mu^{-1/2} 
\end{eqnarray*}
Hence, combining $(a)$ and $(b)$, we complete the proof.
\end{proof}

\subsection{Proof of Lemma 3 in paper}
Lemma 3 in our paper states as follows, 

\textit{Let $\{\epsilon(t)\}$ be an independent (or $E$-valued martingale-difference) sub-Gaussian random variable sequence, with corresponding bounding sequence (i.e., variance proxy sequence) $\{\sigma_t\}$, and $\mathbf x(t)\in \mathbb{R}^{d}, t\geq 1$. Then for $\mu \in (0,1)$, with a probability at least $1-\frac{\delta}{2}$, we have 
\begin{equation*}
\small
\left\|  \sum_{k=1}^t (1-\mu)^{t-k}\mathbf{x}(k)\epsilon(k)\right\| \leq 2\sqrt{2}\left(1+\sqrt{3 \ln\frac{2t}{\delta}}\right)\sigma^* \mu^{-\frac{1}{2}}
\end{equation*}
where $\sigma^* = \sup_{k=1,\cdots,t}\lVert \mathbf{x}(k)\rVert \cdot \sup_{k=1,\cdots,t}\sigma_k $.}

\begin{proof}
First, we examine the "light tail" condition, i.e.,condition $(\mathcal{C}_\alpha[\sigma^\infty])$ in Theorem 2.1 proposed in \cite{journal/juditsky2008large}, let $\sigma'^2_i = (\max \lVert \mathbf{x}(k) \rVert^2)\sigma^2_i, \forall i\geq 1$, then
\begin{eqnarray*}
\mathbb{E}_{i-1}\left\{ \exp\{\frac{\lVert \epsilon(k) \mathbf{x}(k) \rVert^2}{\sigma'^2_i}\} \right\}& \leq &\mathbb{E}_{i-1}\left\{ \exp\{\frac{\lVert \epsilon(k) \rVert^2 \lVert \mathbf{x}(k) \rVert^2}{\sigma'^2_i}\} \right\} \\
&= & \mathbb{E}_{\lVert \mathbf{x}(k) \rVert}\mathbb{E}_{ \epsilon(k)} \left\{ \exp\{\frac{\lVert \epsilon(k) \rVert^2 \lVert \mathbf{x}(k) \rVert^2}{\sigma^2_i(\max \lVert \mathbf{x}(k) \rVert^2) }\} \right\}\\
&\leq & \mathbb{E}_{\lVert \mathbf{x}(k) \rVert}\mathbb{E}_{ \epsilon(k)}  \exp\{\frac{\lVert \epsilon(k) \rVert^2}{\sigma^2_i}\} \leq  \exp\{1\} 
\end{eqnarray*}
Hence, conclusion in Lemma \ref{lemma:concentration} holds on sequence $\{\mathbf{x}(t)\epsilon(t)\}$. Similar to the proof in proof of Lemma 2(original paper), by taking $Y(t)\equiv Z(t) \equiv I$, we complete the proof.
\end{proof}

\subsection{Proof of Theorem 2}
Now let's begin to prove the main theorem, i.e., Theorem 2 in our paper states as follows.
\begin{thm}
\textit{Assume following conditions be satisfied:\vspace{-3mm}
\begin{itemize}
\item[(\expandafter{\romannumeral1})] drift term $\{ \mathbf{s}(t)\}$ is an $E$-valued martingale-difference sub-Gaussian random vector sequence, with corresponding bounding sequence $\{\gamma_t\}$; \vspace{-1mm}
\item[(\expandafter{\romannumeral2})] output noise $\{\epsilon(t)\}$ is an independent (or $E$-valued martingale-difference) sub-Gaussian random variable sequence, with corresponding bounding sequence (i.e., variance proxy sequence) $\{\sigma_t\}$.\vspace{-1mm}
\end{itemize}
Then with a probability at least $1-\delta$, we have 
\begin{equation*}
\label{error-bound}
\| \mathbf{w}(t)-\hat{\mathbf{w}}(t)\rVert \leq K \left\{ (1-\mu)^t\lVert R(0) \rVert \lVert \tilde{\mathbf{w}}(0) \rVert + \sqrt{2} (1+\sqrt{3 \ln (2t/\delta)}) \cdot [2\sigma^*\mu^{1/2} + \gamma^*(\|R(0)\|+{x^*}^2) \mu^{-1/2}]\right\}
\end{equation*}
where $K = \sup_{k=1,\cdots,t}\lVert P(k)\rVert$, $x^* = \sup_{k=1,\cdots,t}\lVert \mathbf{x}(k)\rVert$, $\sigma^* = \sup_{k=1,\cdots,t} \lVert \mathbf{x}(k)\rVert \cdot \sup_{k=1,\cdots,t}\sigma_k$ and $\gamma^* = \sup_{k=1,\cdots,t}\gamma_k$.}
\end{thm}

\proof
Define $R(t)\triangleq P^{-1}(t)$, then we have 
\[ 
\begin{split}
\Vert \bm{\tilde{\mathbf{w}}}(t) \Vert &=  \Vert P(t)R(t)\tilde{\mathbf{w}}(t) \Vert \leq \Vert P(t) \| \|R(t)\tilde{\mathbf{w}}(t)\| \leq K \|R(t)\tilde{\mathbf{w}}(t)\|
\end{split}
\]

so we only need to consider $\|R(t)\tilde{\mathbf{w}}(t)\|$. Recall model assumption, drifting assumption and update rule,  
\begin{equation*}
\begin{split}
&y(t) = \mathbf x(t)^{\rm T} \mathbf w(t-1)+\epsilon(t)\\
&\mathbf{w}(t) = \mathbf{w}(t-1) + \mathbf{s}(t)\\ 
&\hat{\mathbf{w}}(t) = \hat{\mathbf{w}}(t-1) + \mu P(t)\mathbf{x}(t)\left[y(t)-\mathbf x(t)^\mathrm{T}\hat{\mathbf{w}}(t-1)\right]\\
\end{split}
\end{equation*}
we can obtain linear recurrence relations of $R(t)\tilde{\mathbf{w}}(t)$,
\begin{equation*}
R(t)\tilde{\mathbf{w}}(t) = (1-\mu) R(t-1)\tilde{\mathbf{w}}(t-1)+\mu \mathbf{x}(t)\epsilon(t) - R(t)\mathbf{s}(t)
\end{equation*}
which gives 
\begin{equation*}
\begin{split}
R(t)\tilde{\mathbf{w}}(t) &= (1-\mu)^t R(0)\tilde{\mathbf{w}}(0) + \sum_{k=1}^t (1-\mu)^{t-k}[\mu\mathbf{x}(k)\epsilon(k)- R(k)\mathbf{s}(k)]
\end{split}
\end{equation*}
Hence, by Minkowski's inequality, it can be bounded by
\begin{equation}
\label{split}
\left\lVert R(t)\tilde{\mathbf{w}}(t)\right\rVert \leq (1-\mu)^t \left\lVert R(0)\tilde{\mathbf{w}}(0)\right\rVert +  \left\lVert\sum_{k=1}^t (1-\mu)^{t-k}\mu\mathbf{x}(k)\epsilon(k)\right\rVert +  \left\lVert \sum_{k=1}^t (1-\mu)^{t-k} R(k)\mathbf{s}(k)\right\rVert
\end{equation}

When output noise $\{\epsilon(t)\}$ is an independent (or $E$-valued martingale-difference) sub-Gaussian random variable sequence, with corresponding bounding sequence (i.e., variance proxy sequence) $\{\sigma_t\}$, then directly apply Lemma 3(original paper) on the second term in \eqref{split} gives, with at least $1-\frac{\delta}{2}$, 
\begin{equation}
\label{bound-2}
\mu \left\|  \sum_{k=1}^t (1-\mu)^{t-k}\mathbf{x}(k)\epsilon(k)\right\| \leq 2\sqrt{2}\left(1+\sqrt{3 \ln\frac{2t}{\delta}}\right)\sigma^* \mu^{1/2}
\end{equation}
where $\sigma^* = \sup_{k=1,\cdots,t} \lVert \mathbf{x}(k)\rVert \cdot \sup_{k=1,\cdots,t}\sigma_k$.

When the drift term $\{ \mathbf{s}(t)\}$ is an $E$-valued martingale-difference sub-Gaussian random vector sequence, with corresponding bounding sequence $\{\gamma_t\}$. Besides, it's trivial to verified that $R(k)$ meets the decreasing recursive structure stated, then directly apply Lemma 2 on the last term in \eqref{split} gives, with at least $1-\frac{\delta}{2}$, 
\begin{equation}
\label{bound-3}
\left\|  \sum_{k=1}^t (1-\mu)^{t-k}R(k)\mathbf{s}(k)\right\| \leq \sqrt{2}\left(1+\sqrt{3 \ln\frac{2t}{\delta}}\right)\gamma^*(\|R(0)\| +  {x^*}^2) \mu^{-1/2}
\end{equation}
where $x^* = \sup_{k=1,\cdots,t}\lVert \mathbf{x}(k)\rVert$ and $\gamma^* = \sup_{k=1,\cdots,t}\gamma_k$.

Combining \eqref{bound-2} and \eqref{bound-3} by union bound, we get desired estimate error bound in assertion, hence complete the proof.\qed

\section{Additional Experiments}
\label{subsection-4.1}
\subsection{Regression}
In this part, we compare the proposed DFOP to state-of-the-art streaming regression approaches and on both synthesis and real-world datasets. The comparison methods are (a) RLS, least square approach solved in a recursive manner, (b) Online Bagging (OB)~\cite{conf/kdd/OzaR01}, (c) AddExp.C~\cite{conf/icml/KolterM05}, (d) EOS-ELM~\cite{journals/ijon/LanSH09}, (e) $\text{Learn}^{++}$.NSE~\cite{journals/tnn/ElwellP11}, and (f) OAUE~\cite{journals/isci/BrzezinskiS14}, where (b)-(f) are all ensemble style approaches which dynamically adapt the models based on the coming data item or data batch, the differences are the strategies on how to update models like adding new models, excluding models or adjusting models' weights.

Both synthesis dataset \textit{hyperplane} and real-world datasets \textit{Sulfur recovery unit} and \textit{Debutanizer column} are employed to demonstrate the effectiveness of proposed DFOP. The detailed descriptions of datasets are included in Section 4.1 of supplementary material.

The performance of various streaming regression approaches is assessed by MSE (mean square error) between ground-truth and predict values over 10 trails. Both mean and standard deviation are reported in Table \ref{table:Regression}. It has to be pointed out that, not all comparisons are fair. DFOP is a one-pass approach which needs to scan each data item only once. For most ensemble style methods, data items are usually scanned many times for the need of dynamically update ensemble models. Besides, $\text{Learn}^{++}$.NSE is batch basis, nevertheless, DFOP is incremental and could update the model by a single data item. 
\begin{table}[h]
\centering
\small
\caption{Mean and standard deviation of MSE over 10 trails, the lower the better, and the all the values have been multiplied by 100. Besides, $\bullet$ ($\circ$) indicates that DFOP is significantly better (worse) than the compared method (paired $t$-tests at 95\% significance level). Note that the comparisons are unfair to DFOP, because DFOP is able to only scan data once, whereas most compared methods not.}
\label{table:Regression}\vspace{2mm} 
\begin{tabular}{c|cccc}\hline
Methods 	&	 \multicolumn{1}{c}{hyperplane}					&	 \multicolumn{1}{c}{SRU-1}		   		&	 \multicolumn{1}{c}{SRU-2}		   				&	 \multicolumn{1}{c}{Debutanizer}					\\ \hline
RLS			&  2.291(.045)$\bullet$ 	&	0.306(.006)$\bullet$	&	0.392(.005)$\bullet$ 	&	2.467(.003)$\bullet$    \\
OB			&  1.914(.004)$\bullet$ 	&	0.051(.000) 			&	0.142(.000)$\bullet$ 	&	0.577(.000)$\bullet$    \\
AddExp.C	&  0.730(.002)$\bullet$ 	&	0.049(.001)$\circ$ 		&	0.139(.001)$\bullet$ 	&	1.283(.000)$\bullet$    \\
EOS-ELM		&  1.903(.002)$\bullet$ 	&	0.050(.000)			 	&	0.142(.001)$\bullet$ 	&	1.057(.000)$\bullet$    \\
Learn++.NSE	&  1.662(.009)$\bullet$ 	&	0.098(.028)$\bullet$ 	&	0.232(.024)$\bullet$ 	&	1.040(.000)$\bullet$    \\
OAUE		&  1.959(.002)$\bullet$ 	&	\textbf{0.047(.000)}$\circ$ 		&	0.131(.001)$\bullet$ 	&	0.731(.008)$\bullet$    \\
DFOP		&  \textbf{0.619(.048)} 	&	0.063(.004)		&	\textbf{0.064(.004)} 	&	\textbf{0.360(.057)}  	\\  \hline
\end{tabular} 
\end{table}

From Table \ref{table:Regression}, we can see that DFOP outperforms RLS, and achieved a satisfying performance with lowest MSE in three datasets. Meanwhile, since the real drifting property of synthesis dataset is clear, we also present the trend chart of DFOP with different forgetting factors in terms of square loss over generated test data and  estimate error $\lVert \tilde{\mathbf{w}}(t)\rVert = \lVert \mathbf{w}(t)-\hat{\mathbf{w}}(t)\rVert$ in Figure \ref{figure:regression}. And we can see that the error of DFOP is significantly lower than RLS, which is a special case of DFOP when forgetting factor is 0. This proves the effectiveness of forgetting mechanism. Besides, in the right side, we can see that $\lVert \tilde{\mathbf{w}}(t)\rVert$ is tending towards stability as time series go to the infinity which is consistent with our theory proposed in the original paper even the dataset does not exactly meet the assumptions proposed in Theorem 2.
\begin{figure}[t]
\centering
\includegraphics[width=8.4cm]{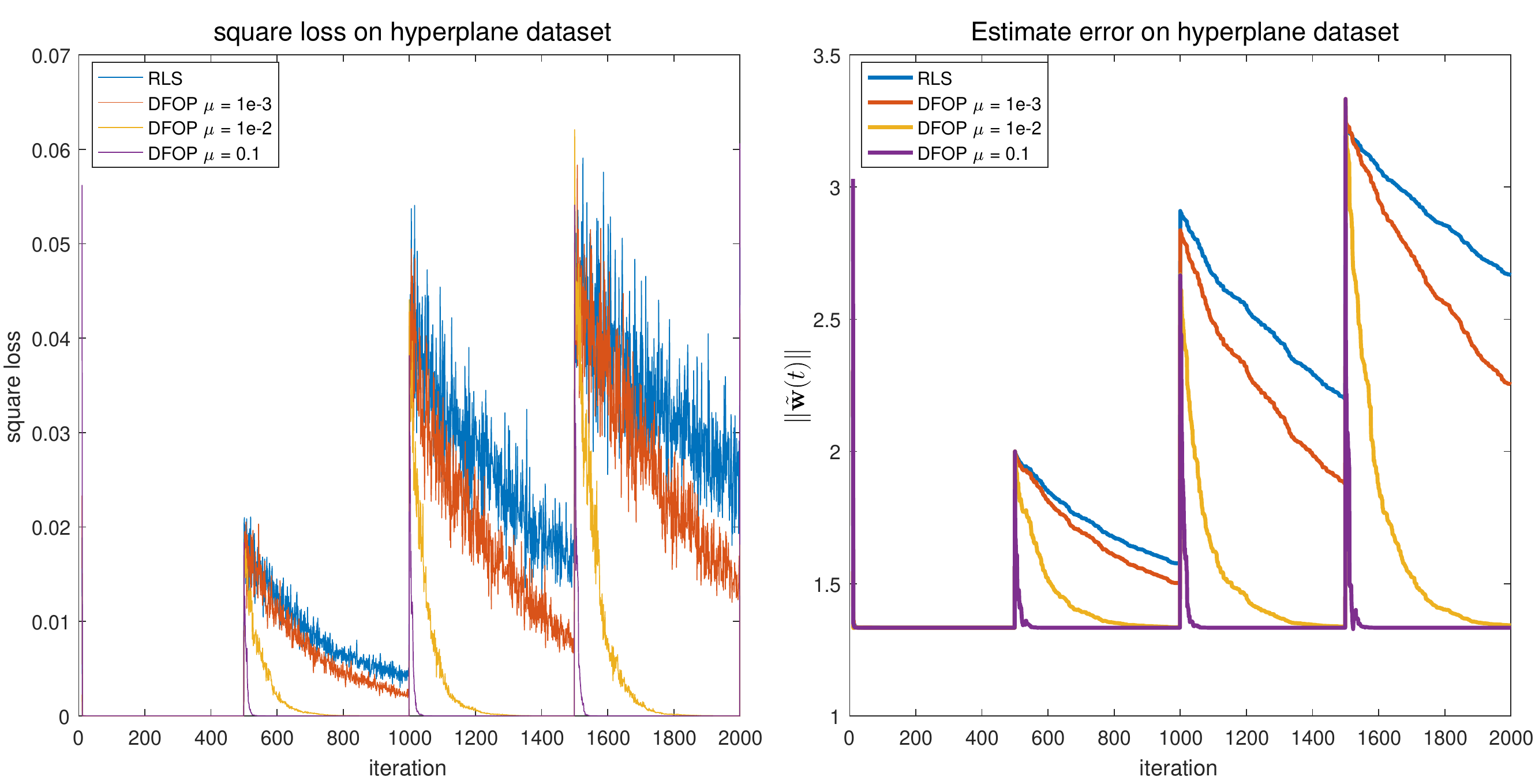}
\caption{A comparison result in terms of square loss over generated test data and $\lVert \tilde{\mathbf{w}}(t)\rVert$ of DFOP with different forgetting factors on hyperplane dataset.}
\label{figure:regression}
\end{figure}

\subsection{Robustness Comparisons}
Additionally, the robustness of all these different algorithms are compared. Concretely speaking, for a particular algorithm \textit{algo}, similar to definition in \cite{conf/kdd/VlachosDGKK02}, the robustness here is defined as the proportion between its accuracy and the smallest accuracy among all compared algorithms,
\[
r_{algo} = \frac{acc_{algo}}{\min_{\alpha} acc_\alpha}
\]
Apparently, the worst algorithm has $r_{algo}=1$, and the others have $r_{algo}\geq 1$, the greater the better. Hence, the sum of $r_{algo}$ over all datasets indicates the robustness of for algorithm $algo$. The greater the value of the sum, the better the performance of the algorithm.

We provide a robustness comparison on six compared algorithms and DFOP over 20 datasets in Figure \ref{figure:robustness}. From the figure, we can see that DFOP achieves the best over 20 datasets, and RLS ranks last as expected since it didn't consider the evolving distribution in datasets at all. For sliding window type approaches, i.e., SVM-win and 1NN-win, their performances are not quite satisfying, even if we have chosen a relatively optimal window size by cross-validation with different data splits. For two batch type approaches, i.e., SVM-fix and SVM-ada. Performance of SVM-fix is comparable to DFOP, however, it's curious that SVM-ada demonstrates a frustrating performance, even worse than SVM-fix. This might because the estimate generalization error is not consistent with empirical one. DWM shows a competitive performance, nevertheless, it needs to scan data many times to maintain a dynamic expert pool. As the only one-pass learning algorithm above except for RLS, DFOP shows a satisfying performance and property in handling data stream with distribution change.
\begin{figure}[t]
\centering
\includegraphics[width=0.28\textwidth]{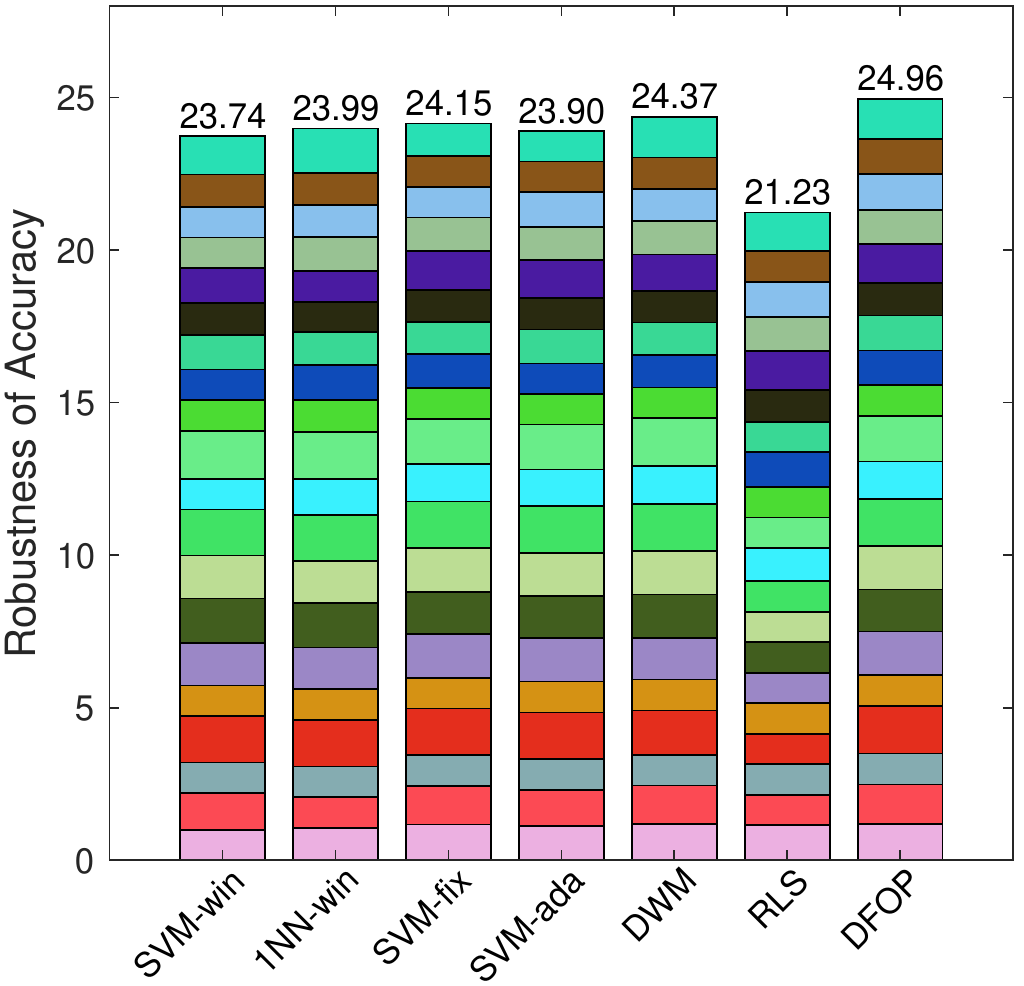}
\caption{Robustness of accuracy on six compared algorithms and DFOP over 20 datasets with distribution change.}
\label{figure:robustness}
\end{figure}

\section{Detailed Descriptions of Datasets}
\subsection{Regression Datasets}
\begin{itemize}
\item[-] \textit{hyperplane}~\cite{conf/icml/KolterM05} is a benchmark synthesis dataset for regression scenario, where feature consists of 10 variables, each is uniform sampled from $[0,1]$. There are in total 2,000 data samples with four stages, target concepts are $(x_i + x_{i+1} + x_{i+2})/3>0.5$, with $i=1,2,4,7$ in each stage.\vspace{-1mm}
\item[-] \textit{Sulfur recovery unit}~\cite{book/springer/fortuna2007soft} is, a real-world dataset which is a record of gas diffusion. Feature consists of 5 different chemical and physical indexes, with in total 10,081 data samples. There are two outputs in original dataset represent the concentration of $\text{SO}_2$ and $\text{H}_2\text{S}$, and we split it into SRU-1 and SRU-2, respectively.  
\vspace{-1mm}
\item[-] \textit{Debutanizer column}~\cite{book/springer/fortuna2007soft} is, a real-world dataset which is a record of chemical reactions, with in total 2,394 data samples consisting of 7 different features. The output represents C4 content in the debutanizer bottoms.     
\end{itemize}

\subsection{Classification Datasets}
\begin{itemize}
\item[-] \textit{Chess}~\cite{journals/ida/Zliobaite11} is constructed by the data from chess.com portal, which consists of game records of players over a period from December 2007 to March 2010, a total of 533 instances with 7 attributes. The label indicates win or loss in the game.
\item[-] \textit{Usenet}~\cite{conf/ecai/KatakisTV08} is split into \textit{Usenet-1} and \textit{Usenet-2} which both consist of 1,500 instances with 100 attributes based on 20 newsgroups collection. They simulate a stream of messages from different newsgroups that are sequentially presented to a user, who then labels them according to his/her personal interests. 
\item[-] \textit{Luxembourg}~\cite{journals/ida/Zliobaite11} is constructed using European Social Survey data. There are 1,900 instance with 32 attributes in total, each instance is an individual and attributes are formed from answers to the survey questionnaire. The label indicates high or low internet usage. 
\item[-] \textit{Spam Detection}~\cite{journals/jiis/KatakisTBBV09}, real-world textual dataset that use email messages from the Spam Assassin Collection, and boolean bag-of-words approach is adopted to represent emails. It consists of 9,324 instances with 500 attributes, and label indicates spam or legitimate.
\item[-] \textit{Weather}~\cite{journals/tnn/ElwellP11} Weather dataset is originally collected from the Offutt Air Force Base in Bellevue, Nebraska. 18,159 instances are presented with an extensive range of 50 years ($1949-1999$) and diverse weather patterns. Eight features were selected based on their availability, eliminating those with a missing feature rate above 15\%. The remaining missing values were imputed by the mean of features in the preceding and following instances. Class labels are based on the binary indicator(s) provided for each daily reading of rain with 18,159 daily readings: 5698 (31\%) positive (rain) and 12,461 (69\%) negative (no rain).
\item[-] \textit{Powersupply}~\cite{journals/archive/chen2015ucr} contains three year power supply records including 29,928 instances with 2 attributes from 1995 to 1998, and our learning task is to predict which hour the current power supply belongs to. We relabeled into binary classification according to p.m. or a.m.
\item[-] \textit{Electricity}~\cite{journal/harries1999splice} is wildly adopted and is collected from the Australian New South Wales Electricity Market where prices are affected by demand and supply of the market. The dataset contains 45,312 instances with 8 features. The class label identifies the change of the price relative to a moving average of the last 24 hours. 
\end{itemize} 

The basic information for all datasets mentioned in the original paper is summarized in the following Table \ref{table:Classification}.
\begin{table*}[h] 
\centering
\caption{Basic information of various datasets with distribution change}\vspace{2mm}\label{table:Classification}
\begin{tabular}{l|rc||l|rc}\hline
Dataset   	& \#instance & \#dim   & Dataset & \#instance & \#dim         \\ \hline
SEA       	& 50,000      & 3     & Chess       & 533         & 7    \\
hyperplane	& 90,000      & 10    & Usenet-1    & 1,500       & 100  \\
1CDT      	& 16,000      & 2     & Usenet-2    & 1,500       & 100  \\
2CDT      	& 16,000      & 2     & Luxembourg  & 1,900       & 32   \\
1CHT      	& 16,000      & 2     & Spam        & 9,324       & 500  \\
2CHT      	& 16,000      & 2     & Weather     & 18,159      & 8    \\
1CSurr    	& 55,283      & 2     & Powersupply & 29,928      & 2   \\
UG-2C-2D    & 100,000     & 2     & Electricity & 45,312      & 8    \\
UG-2C-3D    & 200,000     & 3     & MG-2C-2D    & 200,000     & 2  \\
UG-2C-5D    & 200,000     & 5     & G-2C-2D 	& 200,000     & 2  \\ \hline
\end{tabular}
\end{table*}
\bibliography{reference}

\newcommand{\etalchar}[1]{$^{#1}$}
\begin{thebibliography}{dSSGB15}

\bibitem[BP10]{journals/buldygin2010inequalities}
V~Buldygin and E~Pechuk.
\newblock Inequalities for the distributions of functionals of sub-gaussian
  vectors.
\newblock {\em Theory of Probability and Mathematical Statistics}, 80:25--36,
  2010.

\bibitem[BS14]{journals/isci/BrzezinskiS14}
Dariusz Brzezinski and Jerzy Stefanowski.
\newblock Combining block-based and online methods in learning ensembles from
  concept drifting data streams.
\newblock {\em Information Sciences}, 265:50--67, 2014.

\bibitem[CBL06]{book/Cambridge/cesa2006prediction}
Nicolo Cesa-Bianchi and G{\'a}bor Lugosi.
\newblock {\em Prediction, learning, and games}.
\newblock Cambridge university press, 2006.

\bibitem[CKH{\etalchar{+}}15]{journals/archive/chen2015ucr}
Yanping Chen, Eamonn Keogh, Bing Hu, Nurjahan Begum, Anthony Bagnall, Abdullah
  Mueen, and Gustavo Batista.
\newblock The ucr time series classification archive.
\newblock {\em URL www.cs.ucr. edu/\~{}eamonn/time\_series\_data}, 2015.

\bibitem[dSSGB15]{conf/sdm/SouzaSGB15}
Vin{\'{\i}}cius M.~A. de~Souza, Diego~Furtado Silva, Jo{\~{a}}o Gama, and
  Gustavo E. A. P.~A. Batista.
\newblock Data stream classification guided by clustering on nonstationary
  environments and extreme verification latency.
\newblock In {\em SDM}, pages 873--881, 2015.

\bibitem[EP11]{journals/tnn/ElwellP11}
Ryan Elwell and Robi Polikar.
\newblock Incremental learning of concept drift in nonstationary environments.
\newblock {\em {IEEE} Transactions on Neural Networks}, 22(10):1517--1531,
  2011.

\bibitem[Fan04]{conf/kdd/Fan04}
Wei Fan.
\newblock Systematic data selection to mine concept-drifting data streams.
\newblock In {\em KDD}, pages 128--137, 2004.

\bibitem[FGRX07]{book/springer/fortuna2007soft}
Luigi Fortuna, Salvatore Graziani, Alessandro Rizzo, and Maria~Gabriella
  Xibilia.
\newblock {\em Soft sensors for monitoring and control of industrial
  processes}.
\newblock Springer Science \& Business Media, 2007.

\bibitem[GLP93]{journals/control/guo1993performance}
Lei Guo, Lennart Ljung, and Pierre Priouret.
\newblock Performance analysis of the forgetting factor rls algorithm.
\newblock {\em International journal of adaptive control and signal
  processing}, 7(6):525--538, 1993.

\bibitem[GMCR04]{conf/sbia/GamaMCR04}
Jo{\~{a}}o Gama, Pedro Medas, Gladys Castillo, and Pedro~Pereira Rodrigues.
\newblock Learning with drift detection.
\newblock In {\em Advances in Artificial Intelligence - {SBIA} 2004, 17th
  Brazilian Symposium on Artificial Intelligence}, pages 286--295, 2004.

\bibitem[GZB{\etalchar{+}}14]{journals/csur/GamaZBPB14}
Jo{\~{a}}o Gama, Indre Zliobaite, Albert Bifet, Mykola Pechenizkiy, and
  Abdelhamid Bouchachia.
\newblock A survey on concept drift adaptation.
\newblock {\em ACM Computing Surveys}, 46(4):44:1--44:37, 2014.

\bibitem[Hay08]{book/Pearson/haykin2008adaptive}
Simon~S Haykin.
\newblock {\em Adaptive filter theory}.
\newblock Pearson Education India, 2008.

\bibitem[HW99]{journal/harries1999splice}
Michael Harries and New~South Wales.
\newblock Splice-2 comparative evaluation: Electricity pricing.
\newblock {\em Tech. rep. South Wales Univ}, 1999.

\bibitem[JN08]{journal/juditsky2008large}
Anatoli Juditsky and Arkadii~S Nemirovski.
\newblock Large deviations of vector-valued martingales in 2-smooth normed
  spaces.
\newblock {\em arXiv preprint arXiv:0809.0813}, 2008.

\bibitem[Kli04]{journals/ida/Klinkenberg04}
Ralf Klinkenberg.
\newblock Learning drifting concepts: Example selection vs. example weighting.
\newblock {\em Intelligent Data Analysis}, 8(3):281--300, 2004.

\bibitem[KM03]{conf/icdm/KolterM03}
Jeremy~Z. Kolter and Marcus~A. Maloof.
\newblock Dynamic weighted majority: {A} new ensemble method for tracking
  concept drift.
\newblock In {\em ICDM}, pages 123--130, 2003.

\bibitem[KM05]{conf/icml/KolterM05}
Jeremy~Z. Kolter and Marcus~A. Maloof.
\newblock Using additive expert ensembles to cope with concept drift.
\newblock In {\em ICML}, pages 449--456, 2005.

\bibitem[KM07]{journals/jmlr/KolterM07}
J.~Zico Kolter and Marcus~A. Maloof.
\newblock Dynamic weighted majority: An ensemble method for drifting concepts.
\newblock {\em JMLR}, 8:2755--2790, 2007.

\bibitem[KTB{\etalchar{+}}09]{journals/jiis/KatakisTBBV09}
Ioannis Katakis, Grigorios Tsoumakas, Evangelos Banos, Nick Bassiliades, and
  Ioannis~P. Vlahavas.
\newblock An adaptive personalized news dissemination system.
\newblock {\em Journal of Intelligent Information Systems}, 32(2):191--212,
  2009.

\bibitem[KTV08]{conf/ecai/KatakisTV08}
Ioannis Katakis, Grigorios Tsoumakas, and Ioannis~P. Vlahavas.
\newblock An ensemble of classifiers for coping with recurring contexts in data
  streams.
\newblock In {\em {ECAI} 2008 - 18th European Conference on Artificial
  Intelligence}, pages 763--764, 2008.

\bibitem[KZ09]{journals/ida/KunchevaZ09}
Ludmila~I. Kuncheva and Indre Zliobaite.
\newblock On the window size for classification in changing environments.
\newblock {\em Intelligent Data Analysis}, 13(6):861--872, 2009.

\bibitem[Lit87]{journals/ml/Littlestone87}
Nick Littlestone.
\newblock Learning quickly when irrelevant attributes abound: {A} new
  linear-threshold algorithm.
\newblock {\em Machine Learning}, 2(4):285--318, 1987.

\bibitem[LSH09]{journals/ijon/LanSH09}
Yuan Lan, Yeng~Chai Soh, and Guang{-}Bin Huang.
\newblock Ensemble of online sequential extreme learning machine.
\newblock {\em Neurocomputing}, 72(13-15):3391--3395, 2009.

\bibitem[OR01]{conf/kdd/OzaR01}
Nikunj~C. Oza and Stuart~J. Russell.
\newblock Experimental comparisons of online and batch versions of bagging and
  boosting.
\newblock In {\em KDD}, pages 359--364, 2001.

\bibitem[Ros58]{journal/rosenblatt1958perceptron}
Frank Rosenblatt.
\newblock The perceptron: A probabilistic model for information storage and
  organization in the brain.
\newblock {\em Psychological review}, 65(6):386, 1958.

\bibitem[SK01]{conf/kdd/StreetK01}
W.~Nick Street and YongSeog Kim.
\newblock A streaming ensemble algorithm {(SEA)} for large-scale
  classification.
\newblock In {\em KDD}, pages 377--382, 2001.

\bibitem[SLS99]{conf/kdd/SyedLS99a}
Nadeem~Ahmed Syed, Huan Liu, and Kah~Kay Sung.
\newblock Handling concept drifts in incremental learning with support vector
  machines.
\newblock In {\em KDD}, pages 317--321, 1999.

\bibitem[VDG{\etalchar{+}}02]{conf/kdd/VlachosDGKK02}
Michail Vlachos, Carlotta Domeniconi, Dimitrios Gunopulos, George Kollios, and
  Nick Koudas.
\newblock Non-linear dimensionality reduction techniques for classification and
  visualization.
\newblock In {\em KDD}, pages 645--651, 2002.

\bibitem[Ver10]{journal/vershynin2010intro}
Roman Vershynin.
\newblock Introduction to the non-asymptotic analysis of random matrices.
\newblock {\em arXiv preprint arXiv:1011.3027}, 2010.

\bibitem[WBSD16]{conf/nips/WuBSD16}
Shanshan Wu, Srinadh Bhojanapalli, Sujay Sanghavi, and Alexandros~G. Dimakis.
\newblock Single pass {PCA} of matrix products.
\newblock In {\em NIPS, Barcelona, Spain}, pages 2577--2585, 2016.

\bibitem[Zli11]{journals/ida/Zliobaite11}
Indre Zliobaite.
\newblock Combining similarity in time and space for training set formation
  under concept drift.
\newblock {\em Intelligent Data Analysis}, 15(4):589--611, 2011.

\end{thebibliography}
\bibliographystyle{alpha}
\end{document}